\documentclass[review]{elsaarticle}

\usepackage{booktabs}       
\usepackage{amsfonts}       
\usepackage{nicefrac}       

\usepackage{comment,amsfonts,mathtools,bm,booktabs,lmodern, microtype,xcolor,algorithm,graphicx, amsmath,amssymb, amsthm}
\usepackage{enumerate}
\usepackage[noend]{algorithmic}
\usepackage[colorlinks=false,linkbordercolor = {white},hidelinks,breaklinks=true]{hyperref}

\usepackage{boldline}

\def\argmax{\mathop{\rm arg\,max}}

\newtheorem{definition}{Definition}[section]
\newtheorem{remark}{Remark}[section]

\newtheorem{theorem}{Theorem}[section]
\newtheorem{lemma}{Lemma}[section]
\newtheorem{corollary}{Corollary}[section]

\newtheorem{proposition}{Proposition}[section]

\def\noofprod{n}
\def\noofset{N}
\newcommand{\alsh}{\textsc{Assort-MNL(BZ)}}
\newcommand{\aheu}{\textsc{Assort-MNL(Approx)}} 
\newcommand{\ann}{\textsc{Assort-MNL}}
\DeclareMathOperator*{\apargmax}{\mathit{approx \ arg\,max}}

\newcommand{\cS}{\mathcal{S}}
\newcommand{\bv}{\textbf{v}}
\newcommand{\bz}{\textbf{z}}
\newcommand{\bp}{\textbf{p}}
\newcommand{\bu}{\textbf{u}}
\newcommand{\ba}{\textbf{a}}
\newcommand{\revsv}{R_{\bv}(S)}
\newcommand{\revstarv}{R_{\bv}(S^*)}
\newcommand{\revshatv}{R_{\bv}(\hat{S})}
\newcommand{\rev}{R}
\newcommand{\approxsol}{\hat{ \bz}^{\tilde{S}} } 
\newcommand{\exactsol}{ \bz^{\tilde{S}} } 
\newcommand{\failprob}{P_e}
\newcommand{\cnn}{$(c,\failprob)$-NN}
\newcommand{\nonenn}{$(1+\nu,\failprob)$-NN}
\newcommand{\nmips}{$(\nu,\failprob)$-MIPS}

\newcommand{\revision}[1]{\textcolor{black}{#1}} 

\journal{}

\biboptions{authoryear}

\begin{document}

\begin{frontmatter}

\title{Optimizing Revenue while showing Relevant Assortments at Scale}

\author[mymainaddress]{Theja Tulabandhula\corref{mycorrespondingauthor}}
\cortext[mycorrespondingauthor]{Corresponding author}
\ead{tt@theja.org}

\author[mysecondaryaddress]{Deeksha Sinha}
\ead{deeksha@mit.edu}

\author[mymainaddress]{Saketh Karra}
\ead{skarra7@uic.edu}

\address[mymainaddress]{Information and Decision Sciences, University of Illinois at Chicago, IL 60607, USA}
\address[mysecondaryaddress]{Operations Research Center, Massachusetts Institute of Technology, MA 02139, USA}

\begin{abstract}
Scalable real-time assortment optimization has become essential in e-commerce operations due to the need for personalization and the availability of a large variety of items. While this can be done when there are simplistic assortment choices to be made, the optimization process becomes difficult when imposing constraints on the collection of relevant assortments based on insights by store-managers and historically well-performing assortments. We design fast and flexible algorithms based on variations of binary search that find the (approximately) optimal assortment in this difficult regime. In particular, we revisit the problem of large-scale assortment optimization under the multinomial logit choice model without any assumptions on the structure of the feasible assortments. We speed up the comparison steps using advances in similarity search in the field of information retrieval/machine learning.  For an arbitrary collection of assortments, our algorithms can find a solution in time that is sub-linear in the number of assortments, and for the simpler case of cardinality constraints - linear in the number of items (existing methods are quadratic or worse). Empirical validations using a real world dataset (in addition to experiments using semi-synthetic data based on the Billion Prices dataset and several retail transaction datasets) show that our algorithms are competitive even when the number of items is $\sim 10^5$ ($10\times$ larger instances than previously studied).
\end{abstract}


\end{frontmatter}


\section{Introduction}
\label{sec:introduction}

Assortment optimization~\citep{kok2008assortment,rossi2012bayesian} is the problem of showing an appropriate subset (assortment) of recommendations (or items) to a buyer taking into account their choice (or purchase) behavior, and is a key problem studied in the revenue management literature. There are essentially two aspects that define this problem: (a) the purchase behavior model of the buyer, and (b) the metric that the decision-maker wishes to optimize. Intuitively, the subset shown to the buyer impacts their purchase behavior, which in turn impacts the seller's desired metric such as conversion or revenue. Both the offline~\citep{Vel} and online~\citep{shipra} optimization settings have a wide variety of applications in retail, airline, hotel and transportation industries among others. Many variants of the problem have been extensively studied (hence, we only cite a few representative works) and is an integral part of multiple commercial offerings. 
In this paper, we are interested in advancing the algorithmic approaches for assortment optimization that can enable real-time personalized optimization at scale. Thus, we focus on designing algorithms that are \emph{computationally efficient} as well as \emph{data-driven}. 

To motivate scalability, consider how global e-commerce firms like Flipkart, Amazon or Taobao/TMall display items~\citep{feldman2018taking}. Every aspect of the page displayed to a customer is broken down into modular pieces with different teams responsible for delivering different functionalities.
The assortment/recommendations team may have a budget of at most hundreds of milliseconds 
to display the best possible assortment given the current profile of the customer. \revision{Such stringent requirements necessitate using scalable assortment planning algorithms. Unfortunately, there is very limited research that has explored this gap. For instance, ~\cite{Vel} devise methods that optimize assortments using the Benders decomposition technique in minutes, and demonstrate moderate scalability with a real product line instance (of over 3000 products). There is still an overwhelming need for more approaches in this direction, which is where our proposed methods seek to fill the gap (solving $10\times$ larger instances, with an appropriately chosen choice model, see Section~\ref{sec:experiments})}.

\revision{The second key requirement in assortment planning is that optimization approaches should work well with different types of constraints on the set of feasible assortments. For instance, frequent itemsets~\citep{borgelt2012frequent} discovered using transaction logs can readily give us a collection of high quality candidate \emph{data-driven} assortments, which can then lead to computationally difficult optimization problem instances. In other words, while they readily help us isolate items and bundles that are frequently purchased by customers, optimizing over them at scale is a difficult challenge.  Another prominent source that give rise to unstructured feasible regions are business rules (e.g., rules that dictate that some items should not be displayed next to each other).}
Store managers typically curate arbitrary assortments based on domain knowledge and other exogenous constraints. Unfortunately, none of the existing algorithms (including methods for integer programming) work well when these sets are not compactly represented (by compact, we mean a polynomial-sized description of the collection of feasible assortments; for instance, a polytope in $\noofprod$ dimension with at most a polynomial number of facets, where $\noofprod$ is the number of items).

In this work, we focus on one popular parametric single purchase behavior model, namely the multinomial logit (MNL) model, and provide algorithms to maximize expected revenue (assuming fixed known prices) that are both scalable as well as capable of handling an arbitrary collection of assortment candidates. These algorithms, namely \ann,  \aheu \ and \alsh, build on: (a) (noisy) binary search, and (b) make use of efficient data structures for similarity search that is used to solve the comparison steps, both of which have not been used for any type of assortment optimization before. While the use of the MNL model to capture customer behavior may seem restrictive here, they can be made quite flexible by making the underlying parameters functions of rich customer profile information (e.g., these functions can be neural networks). Our work adds to the initial success of MNL based approaches at Alibaba~\citep{feldman2018taking}.

The consequence of these choices is that we can solve problem instances with \emph{extremely general specification of candidate assortments} in a \emph{highly scalable manner}. In particular, by leveraging recent advances in similarity search, a subfield of machine learning/information retrieval, our algorithms can solve fairly large instances ($\sim 10^5$ items or assortments) within reasonable computation times (within seconds, see Section~\ref{sec:experiments}) even when the collection of feasible assortments have no compact representations. This allows store managers to seamlessly add or delete assortments/items, while still being able to optimize for the best assortments to show to the customers. Further, in the capacity-constrained assortment setting, not only do our methods become more time (specifically, \textit{linear} in the number of items) and memory efficient in theory, they are also better empirically as shown in our experiments. 

\revision{
To summarize, the key contributions of this work are:
\begin{itemize}
\item We develop new binary-search based algorithms (\ann,  \aheu \ and \alsh) to solve the assortment optimization problem over arbitrary (data-driven) feasible regions at scale, and provide explicit time-complexity and correctness guarantees (see Table~\ref{tab:cap-summary}).
\item We reduce the original maximization problem into a series of decision problems within a binary search loop, and develop similarity search techniques to solve the decision problems efficiently.
\item We address approximation and incorrectness in solving the comparison step decision problems explicitly in \aheu \ and \alsh{} respectively. In particular, the latter relies on a Bayesian update of the binary search interval, and we provide a high probability guarantee in being able to obtain near optimal revenue. 
\item Our approaches are the first to connect frequent itemset mining (and its variants that mine customer purchase/interaction patterns) to the  problem of recommending/assortment planning.
\item We extensively validate the scalability of our methods, under both the general assortment setting as well as the well known capacitated setting. These experiments are carried out using a real world dataset \citep{taFeng}, as well as semi-synthetic instances created using prices from the Billion Prices dataset~\citep{IAH6Z6_2016} and multiple real world transaction logs~\citep{borgelt2012frequent}. 
\end{itemize}
}

Thus, our work addresses the gap of practical assortment planning at Internet scale (where we ideally seek solutions in 100s-1000s of milliseconds) and is complementary to works such as~\cite{Vel}, which focus on richer choice models (e.g., the distribution over rankings model, Markov chain model etc.) that lead to slower integer programming approaches. This latter class of solutions is also limited to instances where the sets can be efficiently described by a polytope. Fixing the choice model to be the MNL model allows our algorithms to scale to practical instances, especially those arising in the Internet retail/e-commerce settings as described above. 

\begin{table}[]
    \begin{center}
	\resizebox{\textwidth}{!}{
	\begin{tabular}{cccccc}
	\hlineB{2}
    \textbf{Work}& Constraints & Approach & Guarantee & Instance Solved \\
	\hlineB{2}
	- & General & Exhaustive (exact) & $O(nN)$ & 50000 sets \\
	{(this work)} & General & \aheu \ (bin. search, $\epsilon$-opt) & $O\left(\noofprod \noofset^{\rho}  \log \frac{1}{\epsilon} \right)$ & 50000 sets\\
	{(this work)} & General & \alsh \ (bin. search,  $\epsilon$-opt) & $O\left(\noofprod \noofset^{\rho}  \log \frac{1}{\epsilon} \right)$ &  50000 sets \\
	\scriptsize{\citep{Paat}} & Capacitated & \textsc{Static-MNL} (specialized, exact) & $O(n^2\log n)$ & 200 items\\
	\citep{sumida2020revenue} & Capacitated & Linear Programming (exact) & $O(n^{3.5})$ & -\\
	\citep{jagabathula2014assortment} & Capacitated & ADXopt (greedy, exact) & $O(n^2)$ & 15 items\\
	{(this work)} & Capacitated & \ann \ (bin. search,  $\epsilon$-opt) & {$O\left(n\log \frac{1}{\epsilon}\right)$} & {15000} items \\
	\hlineB{2}
	\end{tabular}
	}
	\end{center}
    \caption{Algorithmic approaches for assortment planning under the MNL model, optimizing under arbitrary collection of assortments as well as under capacity constraint. Here, $\rho < 1$, $\epsilon < 1$, $\noofprod$ is the number of items, and $\noofset$ is the number of relevant feasible assortments.}
    \label{tab:cap-summary}
\end{table}{}

The rest of the paper is organized as follows. In Section~\ref{sec:preliminaries}, we describe some preliminary concepts. Our proposed algorithms are in Section~\ref{sec:optimize}, whose performance we empirically validate in Section~\ref{sec:experiments}. Finally, Section~\ref{sec:conclude} presents some concluding remarks and avenues for future work. Several appendices that support these sections with additional details are also provided.

\section{Preliminaries}
\label{sec:preliminaries}
\vspace{2mm}

\revision{ In this section, we introduce the reader to: (a) the assortment planning problem, (b) frequent itemsets as a source for defining relevant feasible assortments to optimize over, and (c) the similarity search problem that allows our algorithms in Section~\ref{sec:optimize} to scale. }

\subsection{Assortment Planning}

The assortment planning problem involves choosing an assortment among a set of feasible assortments ($\mathcal{S}$) that maximizes the expected revenue (note that we use price and revenue interchangeably throughout the paper). Without loss of generality, let the items be indexed from $1$ to $\noofprod$ in the decreasing order of their prices, i.e., $p_1 \geq p_2 \geq \cdots p_{\noofprod}$. Let $\rev(S) = \sum_{l \in S} p_l \times \mathbb{P}(l|S)$ denote the (expected) revenue of the assortment $S \subseteq \{1,...,\noofprod\}$. Here $\mathbb{P}(l|S)$ represents the probability that a user selects item $l$ when assortment $S$ is shown to them and is governed by a single-purchase choice model.

The expected revenue maximization problem is simply: $\max_{S \in \mathcal{S}} \rev(S)$. In the rest of the paper, we focus on the multinomial logit (MNL) model~\citep{luce1960individual} with parameters represented by a vector $\mathbf{v} = \left(v_0, v_1, \cdots v_{\noofprod}\right)$ with $0 \leq v_i \leq 1 \;\;\forall i$. Parameter $v_i, \ 1\leq i \leq \noofprod$, captures the preference of the user for purchasing item $i$. For this model, it can be shown that $\mathbb{P}(l|S) = v_l/(v_0 + \sum_{l' \in S} v_{l'})$.

\subsection{Sources of Relevant Feasible Assortments}

\revision{ There are many sources that can lead to a collection of feasible assortments, with the collection not having a simple enough description (such as being polygonal) for efficient optimization. Here we discuss one such source, namely frequent itemsets, and omit discussing other sources (such as business rules). }

Frequent itemset mining~\cite{han2007frequent} is a well-known data mining technique to estimate statistically interesting patterns from datasets in an unsupervised manner. They were originally designed to analyze retail datasets, in particular, to summarize certain aspects of customer co-purchase behavior with minimal modeling assumptions. Such purchase patterns, for instance, of which items are commonly bought together, can help retailers optimize tasks such as pricing, store design, promotions and inventory planning (their use in assortment planning is novel to this work). If $\noofprod$ is the number of items, then any non-empty subset $I$ of these items is called an itemset. A retail transaction records (near-)simultaneous purchase of items by customers. Let the number of such records be $D>0$ and the corresponding itemsets be $I_1,...,I_D$. A transaction record itemset $I_i$ supports a given itemset $J$ if $ J \subseteq I_i$. The collection of transaction itemsets that support $J$ allow us to define a score for $J$, which is the \emph{support} of $J$: $\textrm{support}(J) = \sum_{i=1}^{D}\mathbf{1}[J \subseteq I_i]/D$.

Itemset $J$ is called a frequent itemset if $\textrm{support}(J) \geq t$, where $t \in (0,1]$ is a pre-defined threshold. Intuitively, the support of the set $J$ can be thought of as the empirical probability of a customer purchasing that set of items. In our setup, computation of such itemsets using support and a myriad of other criteria is part of the pre-processing stage. These sets/assortments are completely data-driven and complement prior-knowledge based assortment candidate designs. Informally, assortments constructed using itemsets contain items that have historically been bought frequently, so there may be a \emph{higher chance} that at least one item from the assortment will be bought, a \emph{feature} that complements the single item purchase pattern assumed under MNL.

\subsection{Maximum Inner Product Search (MIPS) and Nearest Neighbor (NN) Search}

\revision{ We now discuss two problems and their fast solutions that enable us in performing comparisons efficiently in the binary search based approaches detailed in Section~\ref{sec:optimize}. The first problem, namely MIPS, is that of finding the vector in a given set of vectors (points) which has the highest inner product with a query vector. Precisely, for a query vector $q$ and a set of points $P$ (with $N = |P|$), the optimization problem is: $ \max_{x \in P} q  \cdot  x$ (the `$\cdot$' operation stands for inner product). The second related problem, namely the nearest neighbor (NN) search problem, aims to find a vector in a given set of vectors which is closest to the query vector in terms of Euclidean distance. The MIPS problem can be transformed to this problem ~\citep{bachrach2014speeding, neyshabur2015symmetric} under some conditions, and both together form an integral part of our algorithm design. }

\revision{ When there is no structure on $P$, one can solve these(MIPS/NN) optimization problems via a linear scan, which can be quite slow for large problem instances. In the past few years, there has been substantial progress on fast methods for solving approximate versions of these. One such approximate problem is the \cnn \ problem defined below.}

\revision{ 
\begin{definition}
The \cnn \  (approximate \emph{nearest} neighbor) problem with failure probability $\failprob \in (0,1)$ and $c \geq 1$ is to construct a data structure over a set of points $P$ that supports the following query: given query $q$, report a $c$-approximate nearest neighbor of $q$ in $P$ i.e., return any $p'$ such that $d(p',q) \leq c\min_{p \in P}d(p,q)$ with probability $1-\failprob$.
\end{definition}
}

\revision{ 
 In the above, there is no approximation when $c=1$, and there is vanishing error when $\failprob\rightarrow 0$. Many fast sub-linear time solutions exist for solving the \cnn \ problem. For instance, Theorems 2.9 and 3.4 in \cite{har2012approximate} (see~\ref{subsec:LSH} for details) demonstrate data structures that solve the $(c(1+O(\gamma)), P_e)$-NN problem in time $O( n N^{\rho(c)} \log{1/P_e})$ and space $O(n N^{1+\rho(c)}/\gamma )$ for some $ \gamma \in \left( 1/N, \ 1 \right)$ and $ \rho(c) < 1$, where $n$ is the dimension of points in $P$ (and recall that $N=|P|$). In particular,  \cite{andoni2008near} propose techniques such that $\rho(c) = \frac{1}{c^2} + O(\log\log N / \log^{1/3}N)$. For large enough $N$, and for say $c = 2$, $\rho(c) \approx .25$, and thus the query completion time is $\propto nN^{.25}$. This is a significant speed up over exhaustive search which would take $O(nN)$ time.}

\revision{
While there are many ways to similarly define the approximate MIPS problem (see \cite{teflioudi2016exact}), we choose the following definition, which is more suited for the analysis our our algorithms in Section~\ref{sec:optimize}.
\begin{definition}
	The \nmips \ problem with failure probability $\failprob \in (0,1)$ and $\nu \geq 0$ is to construct a data structure over a set of points $P$ that supports the following query: given query $q$,  return any $p'$ such that  $1 + (1+\nu)^2(\max_{p \in P}p\cdot q - 1) \leq p' \cdot q \leq \max_{p \in P}p\cdot q$ with probability $1-\failprob$.
\end{definition}
}

\revision{
In the above, there is no approximation when $\nu=0$, and there is vanishing error when $\failprob\rightarrow 0$. The following Lemma describes the equivalence between the solutions of the \nonenn \ and \nmips \ problems, allowing us to use the fast data structures available for the \cnn \ problem (with $c = 1+\nu$) as is.
\begin{lemma} \label{lemma:nn2mips}
For a set of points $P$ and query $q$ lying on the unit sphere, if $p'$ is a solution to the \nonenn  \ problem, then it is also a solution to the \nmips \ problem.
\end{lemma}
}

\revision{
The assumption that points $P$ lie on the unit sphere can be satisfied with straightforward transformations as long as they are bounded \citep{neyshabur2015symmetric}.}
\revision{Thus, given the above equivalence, we can assume that \nmips \ instances can be solved in time $O( n N^{\rho(c)} \log{1/P_e})$ and space $O(n N^{1+\rho(c)}/\gamma )$ for some $ \gamma \in \left( 1/N, \ 1 \right)$ and $ \rho(c) < 1$, where $n$ is the dimension of points in $P$ and $N$ is the number of points in $P$. As we show in the next section, \nmips \ and its fast solutions allow for scalable design of our algorithms, namely \aheu{} and \alsh.  
}
\section{New Algorithms for Assortment Planning}
\label{sec:optimize}

\revision{In this section, we propose three new algorithms: \ann, \aheu \  and \alsh, with each aiming to find an assortment with revenue that is within an additive tolerance $\epsilon$ of the optimal revenue $\revstarv$ (here $S^*$ is an optimal assortment and $\bv$ denotes the MNL model parameter). All three methods, which work with an arbitrary feasible collection of relevant assortments, transform the original maximization problem into a sequence of decision problems. These decision problems can either be solved exactly or approximately, with the latter achievable under smaller time complexities than the former. In particular, each of these decision problems can be viewed as an instance of the MIPS problem introduced in Section~\ref{sec:preliminaries}. The benchmark algorithm when there is no structure in the feasible region is by default exhaustive search, and we will make comparisons to this benchmark throughout.}

\revision{Our first algorithm \ann{} uses binary search and exact MIPS solvers to optimize assortments. Our second algorithm \aheu \ shows how approximate solutions to the MIPS problem can be used effectively assuming the approximation quality is known. Finally, the third algorithm \alsh \ shows how to use MIPS solvers that lead to incorrect decisions with some unknown probability and still obtain a near optimal assortment. Along the way, we also present an optimized version of \ann \ when the feasible set of assortments is given by capacity constraints and variants thereof (this is a special case).}

\subsection{First Algorithm: \ann{}}\label{subsec:ann}

\ann \ (Algorithm~\ref{alg:ann_outline}) aims to find an $\epsilon$-optimal assortment (i.e., an assortment with revenue within a small interval defined by tolerance parameter $\epsilon$ of the optimal assortment's revenue).  In this algorithm, we search for the revenue maximizing assortment using the binary search procedure. In each iteration, we maintain a search interval and check if there exists an assortment with revenue greater than the mid-point of the search interval. Then, we perform a binary search update of the search interval, i.e., if there exists such an assortment then the lower bound of the search interval is increased to the mid-point (and this assortment is defined as the current optimal assortment). Otherwise, the upper bound is decreased to the mid-point. 

\begin{algorithm}[H]
\caption{\ann{}}
\label{alg:ann_outline}
	\begin{algorithmic}[1] 
		\REQUIRE{ Prices $\{p_i\}_{i=1}^{n}$, model parameter \bv, tolerance parameter $\epsilon$, set of feasible assortments $\cS$ } \\
		\STATE{$L_1 = 0, U_1 = p_1 , t=1, \hat{S} = \{1 \} $}\\ 
		\WHILE{ $U_j - L_j > \epsilon$} 
		\STATE{ $K_j = (L_j + U_j)/2 $} \\
		\IF{$K_j \leq  \max_{S \in \cS } \revsv$ \label{alg:compare-step} } 
		\STATE{$L_{j+1} =  K_j, U_{j+1} = U_j$}
		\STATE{Pick any $\hat{S}  \in \{ S:\revsv \geq K_j \}$}
		\ELSE 
		\STATE{$L_{j+1} = L_j , U_{j+1} = K_j$} 
		\ENDIF 
		\STATE{$j = j+1$}
		\ENDWHILE
		\RETURN{$ \hat{S} $}
	\end{algorithmic}
\end{algorithm}

We continue iterating and narrowing down the search space until its length becomes less than the tolerance parameter $\epsilon$. Crucially, note that the binary search loop is redundant if the comparison step is strengthened to solve the assortment optimization problem itself, as we have done here. The strengthening serves the following purpose: even though the overall time complexity of \ann \ is theoretically increased (see Section~\ref{subsubsec:ann-general}), the comparison can be reformulated into a MIPS instance such that in practice \ann \ can solve for the optimal assortment faster than exhaustive search (see Section~\ref{sec:experiments}).

The search interval starts with the lower and upper bounds ($L_1$ and $U_1$) as $0$ and $p_1$ respectively, where $p_1$ is the  highest price among all items (note that $p_1$ is an upper bound on the revenue of any assortment). 
Before the start of the algorithm, the optimal assortment is initialized to the set $\{ 1 \}$. 
\revision{If all the assortments have revenue less than the tolerance parameter $\epsilon$, then every assortment is $\epsilon$-optimal. Without loss of generality, the algorithm returns the set $\{ 1 \}$ as an $\epsilon$-optimal assortment. }
\revision{Algorithm~\ref{alg:ann_outline} produces an $\epsilon$-optimal assortment as formalized by the following easy to derive lemma with the proof provided in \ref{app:algo_correct}.  }

\begin{lemma} \label{lem:ann_correct}
The assortment returned by \ann \ is $\epsilon$-optimal i.e. $ \revshatv \geq \revstarv -  \epsilon$ and the number of iterations needed is $ \left\lceil \log(p_1/\epsilon) \right \rceil$.
\end{lemma}

\revision{Ideally, if we want to solve the original revenue maximization problem using a sequence of decision problems within a binary search loop, we need to perform a slightly different comparison at each step in \ann, which is to check if there exists an assortment with revenue greater than the mid-point ($K$) of the current search interval.}
As a first key step, we have already changed this existence check to a stronger comparison, namely $K \leq  \max_{S \in \cS } \revsv$, and we refer to this as the \textsc{Compare-Step} (see line~\ref{alg:compare-step} of Algorithm~\ref{alg:ann_outline}). The strength of \ann \ is in its ability to efficiently answer a transformed version of this strengthened comparison.

\revision{The transformation is as follows.}  In every iteration, we needed to check if there exists a set $S \in \cS$  such that: $ K \leq (\sum_{i \in S} p_i v_i)/(v_0 + \sum_{i \in S} v_i) 
\Leftrightarrow K \leq \frac{1}{v_0} \ \sum_{i \in S}  v_i (p_i - K)$.
This is equivalent to evaluating the strengthened comparison: $ K \leq \max_{S \in \cS } \allowbreak \sum_{i \in S}  \allowbreak v_i (p_i - K)/v_0$, allowing us to focus on the transformed optimization problem within the comparison: 
 \vspace{-3mm}
\begin{equation}
\max_{S \in \cS}  \ \sum_{i \in S}  v_i (p_i - K).
\label{opt_prob}
\end{equation}

\revision{Before continuing on, we first discuss the special case when the feasible region has a compact representation (i.e., the capacitated setting) in Section~\ref{subsec:capacity-constrained-modification}. The special case will not need MIPS solvers. After that, we will come back to the general setting  where there is no structure in the feasible region in Section~\ref{subsubsec:ann-general}, and continue discussing the above transformation. In particular, we will relate it to the MIPS problem discussed in Section~\ref{sec:preliminaries}.}

\subsubsection{Assortment Planning with Capacity Constraints} \label{subsec:capacity-constrained-modification}

Consider the well studied capacity-constrained setting with $\cS = \{ S:  |S| \leq C  \}$, where constant $C$ specifies the maximum size of feasible assortments.  The key insight here is that the operation $ \argmax_{S \in \cS}  \ \sum_{i \in S}  v_i (p_i - K)$ can be decoupled into problems of smaller size, each of which can be solved efficiently. This is because the problem can be interpreted as that of finding a set of at most $C$ items that have the highest value $v_i(p_i - K)$, and that this value is positive. Thus, to solve this optimization problem we only need to calculate the value $v_i(p_i - K)$ for each item $i$ and sort these values. This strategy can also be extended to many other capacity-like constraints such as a lower bound on the size of feasible assortments, capacity constraints on subsets of items, and finding assortments near a reference assortment. These are described in ~\ref{app:capacity-extensions}.

\revision{
\begin{lemma} \label{lem:ann_capacity_time}
	The time-complexity of \ann \ for assortment planning with capacity constraints is $O(n\log C \log \frac{p_1}{\epsilon})$ and the space complexity is $O(n)$.
\end{lemma}
}

Note that other algorithms  for solving the capacitated assortment planning problem, such as ADXOpt and \textsc{Static-MNL} have time complexity quadratic in $\noofprod$ because they are performing exact search. Further, note that exhaustive search is not competitive here as its time complexity is O($n^C$).

\subsubsection{Assortment Planning with General Constraints}\label{subsubsec:ann-general}

In the absence of structure in the set of feasible assortments, we cannot decouple the optimization problem in Equation (\ref{opt_prob}) as before. Nonetheless, this can be reduced to a specific MIPS problem instance as described in Algorithm \ref{alg:ann_comp}. Here, $\mathbf{1}\{ \cdot \}$ represents the indicator function and $\textbf{a} \circ \textbf{b}$ represents the Hadamard product between vectors $\textbf{a}$ and $\textbf{b}$ \revision{(i.e., produces a vector by performing an element-wise product of the two input vectors). }

\begin{algorithm}[H]
	\caption{\textsc{Compare step} in \ann{} }
	\label{alg:ann_comp}
	\begin{algorithmic}
		\STATE{Given comparison: 
			$K \leq \max_{S \in \cS }  \sum_{i \in S}  v_i (p_i - K)/v_0$}
		\STATE{Formulate an equivalent MIPS instance with:
			\vspace{-4mm}
			\begin{align*}
			& \bp := \left( p_1, p_2 \cdots p_{\noofprod} \right) ; \ \mathbf{\hat{v}_K} := (v_1, \cdots, v_n, -v_1K,-v_2K, \cdots -v_{\noofprod}K) \\
			& \bu^S := (u_1, u_2, \cdots u_n) \ \text{ where } \  u_i = \mathbf{1} \{ i \in S\}, \text{for any } S \in \cS  \\
			& \hat{\bz}^S := \left( \bp \circ \bu^S, \bu^S \right), \text{ for any } S \in \cS ; \ \mathbf{\widehat{Z}} := \{ \hat{\mathbf{z}}^S : S \in \cS \}
			\end{align*}
		}\\
	\vspace{-3mm} 
		\STATE{Solve the MIPS instance: $ \bz^{\tilde{S}} \in \argmax_{\hat{\bz}^S \in \mathbf{\widehat{Z}}   } \mathbf{\hat{v}_K} \cdot \hat{\bz}^S$}
		\STATE{Output result of an equivalent comparison: $  K \leq  \bv \cdot \bz^{\tilde{S}}/v_0 $}
	\end{algorithmic}
\end{algorithm}

The number of points in the search space of the above MIPS problem is $N = |\cS|$. As discussed in Section~\ref{sec:preliminaries}, this reductions allows for efficient exact and approximate answers to the comparison. In \ann \ we will assume that the MIPS instances are solved exactly. In this case, the following time complexity result holds.
 
 \revision{
 \begin{lemma}
 	The time-complexity of \ann \ for assortment planning with general constraints is $O(\noofprod \noofset \log \frac{p_1}{\epsilon}  )$ and the space complexity is $O(nN)$. \label{lemma:ann-general-time-complexity}
 \end{lemma}
} 

Again, as mentioned earlier, while the above theoretical result is worse than exhaustive search (whose complexity is $\Theta(nN)$), if MIPS is solved using nearest neighbor searches, it turns out that there is a distinct computational advantage to be had. This is because exact nearest neighbor searches in vector spaces are very efficient in practice, and allow \ann \ to have a much better computational performance compared to exhaustive search (see Section~\ref{sec:experiments}).  

\subsection{ Second Algorithm: \aheu }\label{subsec:aheu}

Recall that to solve the MIPS problem introduced in the \ann \ algorithm, we need the nearest point to $\mathbf{\hat{v}_K}$ in the set $\mathbf{\widehat{Z}}$ according to the inner product metric (Algorithm \ref{alg:ann_comp}).  \revision{ Now, instead of finding such a nearest point exactly, we can consider computing it approximately (i.e., it will be \emph{near enough}) and/or probabilistically (i.e., the returned point may not be near enough with some probability). In other words, we can find a \nmips \ solution. We denote the \nmips \ solution in this context with the notation $ \apargmax_{\bz^S \in \mathbf{Z}(K)   }  \bv \cdot \bz^S$, where $\nu$ and $\failprob$ are the approximation and probability of failure parameters respectively (and are determined by the underlying solution technique used to solve \nmips). }

\revision{Solving the \nmips \ problem in place of the MIPS problem leads to two types of inaccuracies in the solution used in the comparison (which naturally depend on the parameters $\nu$ and $\failprob$) .} We will focus on the approximation aspect (near enough versus nearest) in this subsection, and address the possibility of errors in Section~\ref{subsec:alsh}. \revision{Thus, \aheu{} described below will take approximation into account, while \alsh{} in Section~\ref{subsec:alsh} will take errors into account, and for both we will provide formal error guarantees on the revenue of the assortment returned.} 

Algorithm \aheu{} (see Algorithm~\ref{alg:aheu_eff_mod}) is similar to \ann \, except for two changes: (a) the operation $\bz^{\tilde{S}} = \argmax_{\bz^S \in \mathbf{Z}(K)   }  \bv \cdot \bz^S$ is replaced by $\hat{ \bz}^{\tilde{S}} = \apargmax_{\bz^S \in \mathbf{Z}(K)   }  \bv \cdot \bz^S$, and (b) the approximation in the quality of the solution returned is taken into account in the binary search update.

\begin{algorithm}[H]
	\caption{\aheu{}}
	\label{alg:aheu_eff_mod}
	\begin{algorithmic} 
		\REQUIRE{ Prices $\{p_i\}_{i=1}^{n}$, tolerance $\epsilon$, approximation factor $\nu$ } \\
		\STATE{$L_1 = 0, \ U_1 = p_1 , \ t=1,\bp = (p_1, \cdots, p_n)$} 
		\STATE{$\bu^S = (u_1, u_2, \cdots u_n) \ \text{ where } \  u_i = \mathbf{1} \{ i \in S\} , \text{ for any } S \in \cS $}
		\STATE{$\hat{S}  = \{1 \}, \ \mathbf{\widehat{Z}} = \{ \mathbf{\hat{z}}^S | \hat{\bz}^S = \left( \bp \circ \bu^S, \bu^S \right), S \in \cS \}$}
		\WHILE{ $U_j - L_j > \epsilon$} 
		\STATE{ $K_j = \frac{L_j + U_j}{2}$} \\
		\STATE{$\hat{K}_j = 1 + (1+\nu)^2(K_j - 1)$} \\
		\STATE{$\mathbf{\hat{v}_{K_j}} = \small{(v_1, \cdots, v_n, -v_1K_j,-v_2K_j, \cdots -v_nK_j)}$} \\
		\STATE{$ \hat{\bz}^{\tilde{S}_j} = \apargmax_{\mathbf{\hat{z}}^S \in \mathbf{\hat{Z}}} \ \mathbf{\hat{v}_K} \cdot \mathbf{\hat{z}}^S $}
		\IF{$\hat{K}_j >  \frac{ \bv \cdot \hat{\bz}^{\tilde{S}_j}}{v_0}    $}  
		\STATE{$L_{j+1} = L_j , U_{j+1} = K_j$} \\
		\ELSIF{$K_j \leq  \frac{ \bv \cdot \hat{\bz}^{\tilde{S}_j}}{v_0}    $}  
		\STATE{$L_{j+1} =  K_j, U_{j+1} = U_j, \hat{S}  = \tilde{S}_j $}
		\ELSE
		\STATE{$L_{j+1} = \hat{K}_j , U_{j+1} = U_j, \hat{S}  = \tilde{S}_j $}
		\ENDIF 
		\STATE{$j = j+1$}
		\ENDWHILE
		\RETURN{$ \hat{S} $}
	\end{algorithmic}
\end{algorithm}

For any comparison threshold $K$, let $\hat{K} = 1 + (1+\nu)^2(K - 1)$.   Without loss of generality, let $p_1 \leq 1$, which implies that $\hat{K} \leq K$. Because the returned solution $\approxsol$ is approximate, we can only assert that one of the following inequalities is true for any given comparison, giving us the corresponding update for the next iteration:

\begin{itemize}
\item If $\hat{K} \geq \approxsol \cdot \bv$, then $K \geq \exactsol \cdot \bv$, allowing update of the upper bound of the the search region  to $K$.
\item If $K \leq \approxsol \cdot \bv$, then $K \leq \exactsol \cdot \bv$, allowing update of the lower bound of the the search region  to $K$.
\item If $\hat{K} \leq \approxsol \cdot \bv$ and $K \geq \approxsol\cdot \bv$, then $\hat{K} \leq \exactsol \cdot \bv$, allowing update of the lower bound of the the search region  to $\hat{K}$.
\end{itemize}

We give a proof of the update in the first setting in ~\ref{subsec:aheu-update}. The updates in the other two settings are similar. These conditions and the efficient transformation discussed in Section~\ref{subsubsec:ann-general} are summarized in Algorithm \ref{alg:aheu_eff_mod}. This algorithm gives a desired solution as stated in the lemma below (with the proof in \ref{app:algo_correct}).

\begin{lemma} \label{lem:aheu_correct}
With no probabilistic errors in the $\apargmax$ operation i.e. $\failprob=0$, the assortment returned by \aheu \ is $\epsilon$-optimal i.e. $ \revshatv \geq \revstarv -  \epsilon$.
\end{lemma}

Further, the number of iterations needed in \aheu \ depends logarithmically on the desired tolerance parameter, as stated in the following result. 
\begin{lemma}
	The number of iterations for \aheu \ to obtain an $\epsilon$-optimal solution is given by  $\left\lceil \log \frac{p_1}{\epsilon - 2(\nu^2 + 2\nu)} \right \rceil$. \label{lemma:aheu-iteration-complexity}
\end{lemma}

\revision{
Recall that the fast solution to \nmips \ described in Section~\ref{sec:preliminaries} requires $\failprob >0$. Nonetheless, one can still use such a solution in \aheu{} in practice, turning it into a heuristic (i.e., it is now not guaranteed to output an $\epsilon$-optimal solution). In this case, we can get the following computational complexity result.
\begin{corollary}
	The time-complexity of \aheu, assuming the $\apargmax$ operation is solved using \nmips \ (detailed in Section~\ref{sec:preliminaries}), is $O\left(\noofprod \noofset^{\rho(\nu)}  \log \frac{p_1}{\epsilon - 2(\nu^2 + 2\nu)} \right)$ for some $\rho(\nu) < 1$. Further, the space complexity  is $O( \noofprod \noofset^ {(1+\rho(\nu))} )$.
\end{corollary}
}

\revision{Note that both the iteration complexity and the time complexity depend on the specific solution technique used to solve MIPS approximately (i.e., both depend explicitly on $\nu$). In the above result, a specific approximating MIPS solution approach was assumed, and the analysis approach extends to any other suitably chosen \nmips \ solution technique. Next, we explicitly take into account the fact that comparison step's outcome could result in an error if \nmips \ with $\failprob >0$ is used.}

\subsection{Third Algorithm: \alsh{}}\label{subsec:alsh}

As seen in Section~\ref{subsec:ann}, we  narrow our search interval based on the result of the MIPS solution in \ann \ (recall that out treatment throughout Section~\ref{sec:optimize} is in the arbitrary feasible assortments setting).  When \nmips \ is used in lieu of an exact MIPS solver, there is a chance of narrowing down the search interval to an incorrect range. To address this, we build on the BZ algorithm~\citep{burnashev1974interval} to accommodate the possibility of failure (captured by parameter $\failprob$) in the $\apargmax$ operation. 

\revision{To keep the focus on addressing probabilistic errors, we assume that $\nu=0$. This implies that the approximate MIPS solver returns the exact MIPS solution with probability $1-\failprob$. In fact, for the purpose of this discussion, we do not need $\nu$ to be exactly 0 but just small enough so that for any query $q$,  the only point $p'$ that satisfies $ 1+ (1+\nu)^2 \left( \max_{p \in P}p \cdot q - 1 \right) \leq p' \cdot q$ is $\argmax_{p \in P} p\cdot q$. }

The \alsh \ algorithm, building on the BZ procedure, is given in  Algorithm~\ref{alg:alsh}. 
The key difference between the BZ algorithm and a standard binary search is the choice of the decision threshold  point at which a comparison is made in each round. In binary search, we choose the  mid-point of the current search interval. As we cannot rule out any part of the original search interval when we receive noisy answers, in the BZ algorithm, we maintain a distribution on the value of the optimal revenue. In every iteration, we test if $K \leq \max_{S \in \cS}\revsv$ where $K$ is the median of the distribution, and then update the distribution based on the result of the comparison. 

For the analysis of the performance of the BZ based algorithm \alsh \ , we require that the errors in the $\apargmax$ operation in every iteration, are independent of each other. If such a property is not inherently satisfied by the underlying \nmips \ solution approach, multiple copies of the solution approach will need to be run independently. In other words, we create $T$ distinct data structures that solve the \nmips \ problem, where $T$ is the number of desired iterations (we show that this is logarithmic in the desired accuracy level, and hence is not a significant overhead). Finally, note that $\hat{\theta}_T$ can be greater than the optimal revenue, and \alsh \ does not always output an assortment that has revenue greater than or equal to $\hat{\theta}_T$. More fundamentally, we cannot guarantee an $\epsilon$-optimal revenue always, but we show below that we can achieve $\epsilon$-optimality with high probability.\\

\begin{algorithm}
\caption{\alsh{}}
\label{alg:alsh}
\footnotesize
\begin{algorithmic} 
\REQUIRE{ Prices $\{p_i\}_{i=1}^{n}$, tolerance parameter $\epsilon$ such that $p_1 \epsilon^{-1} \in \mathbb{N}$, number of steps $T$, 
parameter $\alpha < 0.5 $ such that $ P_e \leq \alpha $, and let  $\beta := 1 - \alpha$.} \\
\STATE{Posterior $\pi_j : [0,p_1] \rightarrow  \mathbb{R} $ after $j$ stages:  
$$ \pi_j(x) = \sum_{i=1}^{p_1 \epsilon ^{-1}}a_i(j)\mathbf{1}_{I_i}(x), $$
where  $I_1 = [0,\epsilon]$ and $ I_i = (\epsilon (i-1),\epsilon i]$ for  $i \in \{2, \cdots, p_1 \epsilon^{-1} \}$. Let $\ba=[a_1(j), \cdots, a_{p_1 \epsilon^{-1}(j)} ]$.} \\
\STATE{$\bu^S = (u_1, u_2, \cdots u_n) \ \text{ where } \  u_i = \mathbf{1} \{ i \in S\} , \text{ for any } S \in \cS $.}
\STATE{$\bp = (p_1, \cdots, p_n), \ \mathbf{\widehat{Z}} = \{ \mathbf{\hat{z}}^S | \hat{\bz}^S = \left( \bp \circ \bu^S, \bu^S \right), S \in \cS \}$}
\STATE{ Initialize $a_i(0) = {p_1}^{-1} \epsilon \ \forall i , \ \hat{S} = \{ 1\},  j=0$.} 
\WHILE{ $j < T$} 
\STATE{ \textbf{Sample Selection:} Define $u(j)$ such that $\sum_{i=1}^{u(j)-1}a_i(j) \leq \frac{1}{2} \ , \sum_{i=1}^{u(j)}a_i(j) > \frac{1}{2} $}. Let \\
$K_{j+1} = 
\begin{cases}
{p_1}^{-1} \epsilon (u(j)-1)  & \mbox{ with probability }  Q(j), \textrm{ and} \\
{p_1}^{-1} \epsilon u(j) & \mbox{ with probability }  1 - Q(j), \\
\end{cases} 
$ where $Q(j) = \frac{\tau_2(j)}{\tau_1(j)+\tau_2(j)}$, and  \\
$\tau_1(j) = \sum_{i=u(j)}^{p_1\epsilon^{-1}}a_i(j) - \sum_{i=1}^{u(j)-1}a_i(j)$, \\
$\tau_2(j) = \sum_{i=1}^{u(j)}a_i(j) - \sum_{i=u(j)+1}^{p_1 \epsilon ^{-1}}a_i(j)$. 

\vspace{3mm}
\textbf{Noisy Observation:} 
$\mathbf{\hat{v}_{K_{j+1}}} = \small{(v_1, \cdots, v_n, -v_1K_{j+1} ,-v_2K_{j+1}, \cdots -v_nK_{j+1})}$ \\
$ \hat{\bz}^{\tilde{S}_{j+1}} = \apargmax_{\mathbf{\hat{z}}^S \in \mathbf{\hat{Z}}} \mathbf{\hat{v}}_{K_{j+1}} \cdot \mathbf{\hat{z}}^S $ 
\STATE{If $  R_{\bv}(\tilde{S}_{j+1}) > K_{j+1}, \textrm{ set } \hat{S} = \tilde{S}_{j+1}$}. \\
\vspace{3mm}
\textbf{Update posterior:} The posterior is updated using the Bayes rule. Note that $K_{j+1} = p_1^{-1} \epsilon u, u \in \mathbb{N}.$ Define 
$$  h(K_{j+1}) = \mathbf{1}\left\{K_{j+1} \leq  \frac{\mathbf{\hat{v}}_{K_{j+1}} \cdot \hat{\bz}^{\tilde{S}_{j+1}} }{v_0} \right\}, \quad\textrm{ and }\quad \tau = \sum_{i=1}^u a_i(j) - \sum_{ i = u+1}^{p_1 \epsilon^{-1}} a_i(j). $$

For  $i \leq u$, we have the update \\
{$a_i(j+1) = \begin{cases} 
\frac{2\beta}{1+\tau(\beta - \alpha)}a_i(j) & \mbox{if }  h(K_{j+1}) = 0, \textrm{ and} \\
\frac{2\alpha}{1-\tau(\beta - \alpha)}a_i(j) & \mbox{if }  h(K_{j+1}) = 1. \\
\end{cases}
$}

For  $i > u$, we have the update \\
{$a_i(j+1) = \begin{cases} 
\frac{2\alpha}{1+\tau(\beta - \alpha)}a_i(j) & \mbox{if }  h(K_{j+1}) = 0, \textrm{ and} \\
\frac{2\beta}{1-\tau(\beta - \alpha)}a_i(j) & \mbox{if }  h(K_{j+1}) = 1. \\
\end{cases}
$}

$j = j+1$

\ENDWHILE
\vspace{3mm}
\STATE{ $\theta_T$ is defined as the median of the posterior distribution i.e. $ \int_0^{\theta_T }\pi_T(x) = \frac{1}{2} $. }

\RETURN{$ \hat{S},$ and $ \hat{\theta}_T =  \max \left( \theta_T,  \revshatv \right)$} 
\end{algorithmic}
\end{algorithm}

\noindent\textit{Error Analysis of} \alsh: \revision{ We claim that the probability of error (i.e. the revenue given by the algorithm not being $\epsilon$-optimal) for \alsh \  decays exponentially with the number of iterations of the algorithm, stated formally in the below theorem. }

\begin{theorem}\label{thm:BZError}
	After $T$ iterations of \alsh \ (Algorithm \ref{alg:alsh}):
\revision{	$$\footnotesize P(|\hat{\theta}_{T} - \theta^*| > \epsilon) \leq \frac{p_1 -\epsilon}{\epsilon} \ \left(\frac{P_e}{2\alpha} + \frac{1- P_e}{2(1-\alpha)}\right)^{T},$$}
	where $\theta^*$ is the true optimal revenue,  $\hat{\theta}_{T}$ is the median of the posterior after the $T^{th}$ iteration, and  $\alpha<0.5$ is a parameter such that $P_e \leq \alpha$. 
\end{theorem}
\noindent\textit{Proof outline:} The proof of this Theorem is presented in \ref{app:algo_correct}. It builds on an analysis done by \cite{burnashev1974interval} and \cite{castro2006upper}. 
\revision{ We present a brief outline of the proof here. In this algorithm, we divide the overall search interval into bins and  maintain a distribution on the value of the optimal revenue lying in each of these bins. This distribution is updated in every iteration of the algorithm depending on the results of the approximate MIPS query.}

\revision{For the analysis, the key random variable is the odds (as per the current posterior distribution) of the true bin containing the optimal revenue. Accounting for the update mechanism of the algorithm, we derive how this variable changes in every iteration of the algorithm. Then, we obtain a bound on the ratio of the odds in successive iterations. Finally, we tie together the \textit{confidence}, i.e., the  probability that the current median (of the posterior distribution) is close to the true optimal revenue, with the change in ratio of the odds in every iteration. \qed
}

\begin{figure}
	\centering
	\includegraphics[width=0.35\textwidth]{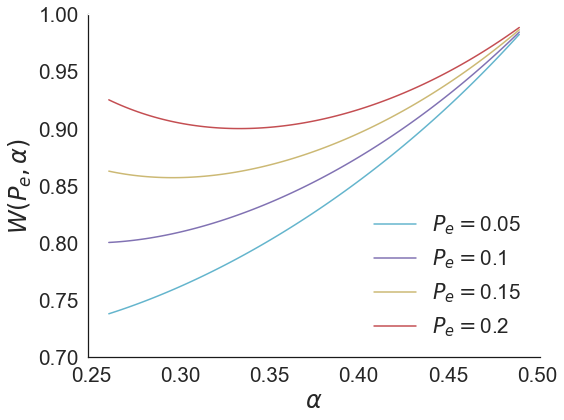}
	\caption{Plot of $W(P_e, \alpha)$ as $\alpha$ varies}
	\label{fig:P_e_alpha}
\end{figure}

\revision{
For a given failure probability $P_e$, the quantity $ W(P_e, \alpha)$ defined as  ${P_e}/{2\alpha} + \allowdisplaybreaks ({1- P_e})/{2(1-\alpha)}$ (plotted in Figure \ref{fig:P_e_alpha}) does not vary monotonically with $\alpha$. In many solution approaches to the \nmips \ problem, the exact value of $P_e$ is not known. This makes it challenging to find an optimal value of $\alpha$ to get fast convergence (note that our analysis does not allow setting $\alpha=\failprob$). Instead, often an upper bound on $P_e$ is available. In the presence of such an upper bound $P_{max}$, we can obtain a tractable expression for the number of iterations needed for any desired confidence level $\gamma$, where $0 < \gamma < 1$ as stated below:}

\revision{
\begin{lemma}
	Any desired confidence level $\gamma$ i.e. $ P(|\hat{\theta}_{T} - \theta^*| > \epsilon) \leq \gamma$ can be achieved with $T= \lceil  \log_{0.5 + \sqrt{P_{max}}} (\gamma \epsilon/(p_1 -\epsilon))   \rceil$ iterations of \alsh, where $P_{max} <  0.25$ is an upper bound on the failure probability $P_e$ of the $\apargmax$ operation.\label{lemma:alsh-time-complexity}
\end{lemma}
}

Similar to \ann \ and \aheu, the number of iterations grows logarithmically in the desired accuracy level for \alsh. 
\revision{
\begin{corollary}
	For a confidence level $\gamma$, the time complexity of \alsh{} \ is $O\left(\noofprod \noofset^ {\rho}\log((p_1 - \epsilon)/\gamma\epsilon)\log(1/\failprob)\right)$ and the space complexity is $O\left(\noofprod \noofset^ {(1+\rho)} \log((p_1 - \epsilon)/\gamma\epsilon) \log(1/\failprob)\right)$ for some $\rho < 1$.
\end{corollary}
}

\revision{Similar to the analysis of \aheu{}, both the iteration complexity and the time complexity depend on the specific solution technique used for \nmips \ (for instance, both depend explicitly on $\failprob$). While a specific approximating MIPS solution approach was assumed in stating above time-complexity, any other suitably chosen \nmips \ technique would also work. Finally, before discussing the empirical performances of the methods, we summarize the guarantees derived in this section in Table~\ref{tab:algo-complex-summary}.
}

\revision{}

\begin{table}[]
	\begin{center}
		\resizebox{\textwidth}{!}{
			\begin{tabular}{ccccc}
				\hlineB{2}
				Constraints & Algorithm & No. of Iterations & Time Complexity & Space Complexity \\
				\hlineB{2}
				Capacity & \ann & $ \lceil \log \frac{p_1}{\epsilon}  \rceil  $ & $O(\noofprod \log C \log \frac{p_1}{\epsilon}  )$  & $O(n)$ \\
				General & \ann &  $ \lceil \log \frac{p_1}{\epsilon}  \rceil  $ & $O(\noofprod N \log \frac{p_1}{\epsilon}  )$  & $O(nN)$ \\
				General & \aheu & $\left\lceil \log \frac{p_1}{\epsilon - 2(\nu^2 + 2\nu)} \right \rceil$ &  $O\left(\noofprod \noofset^{\rho(\nu)}  \log \frac{p_1}{\epsilon - 2(\nu^2 + 2\nu)} \right)$ & $O( \noofprod \noofset^ {(1+\rho(\nu))} )$ \\
				General & \alsh & $\lceil  \log_{1/(0.5 + \sqrt{P_{max}}) } \frac{p_1 -\epsilon}{\gamma \epsilon} \rceil$ & $O\left(\noofprod \noofset^ {\rho}\log \frac{p_1 - \epsilon}{\gamma\epsilon}\log\frac{1}{\failprob} \right)$ & $O(\noofprod \noofset^ {(1+\rho)} \log \frac{p_1 - \epsilon}{\gamma\epsilon}\log\frac{1}{\failprob}  )$ \\
				\hlineB{2}
			\end{tabular}
		}
	\end{center}
	\caption{Table summarizing the runtime and space complexities of the proposed algorithms. They assume the use of an appropriate \nmips \ solver for the $\apargmax$ operation.}
	\label{tab:algo-complex-summary}
\end{table}{}

\section{Experiments}\label{sec:experiments}

\revision{We empirically validate the performance of the proposed algorithms using real and synthetic datasets, when there is no structure on the feasible set as well as in the capacity constrained setting. Before discussing the experimental results, we briefly discuss: (a) the chosen benchmarks and general settings, (b) a modification needed in \aheu{}, and (c) the datasets used.}

\noindent\textit{Benchmarks and Experimental Settings}: Optimizing over a general collection of feasible assortments cannot be efficiently handled by  integer programming formulations (unless there is an efficient representation of the feasible assortments) or other specialized approaches.  Thus, we compare the performance of our algorithms with exhaustive search in this setting. In the capacitated case, we compare these with currently known algorithms for this problem -  \textsc{Static-MNL}~\citep{Paat}, ADXOpt~\citep{jagabathula2014assortment} and a linear programming (LP) formulation~\citep{sumida2020revenue}. \revision{ All experiments are run on a 12 core 32GB 64-bit intel machine (i7-8700 \@ 3.2GHz) with Python 2.7}. We use \textsc{NearestNeighbors} \citep{scikit-learn} for solving MIPS exactly, and  CPLEX 12.7 for solving the LP. \revision{ Performance is measured in terms of the mean computation time and mean relative error in the revenue obtained. The averaging here is over 50 Monte Carlo runs. The tolerance parameter $\epsilon$ is set to $0.1$ for \ann ,  \aheu , and \alsh  \ in all experiments, and a fixed value $\alpha=0.1$ was chosen for \alsh.}

\noindent\textit{Modifying \aheu}: \revision{The approximate \nmips \ solver we use relies on \textsc{LSHForest} \citep{scikit-learn}, which tunes the  approximation parameter $\nu$ in a data-driven way. Hence, to avoid any explicit dependence of this parameter on \aheu \, we redefine it to be the same as \ann \ but with the approximate \nmips \ solver (see Algorithm~\ref{alg:aheu_simpler} in \ref{app:algo}). 
A consequence of this is that the reported performance numbers will be slightly better in terms of time complexity and slightly worse in terms of solution quality and relative error in revenue. As we show below, the redefined \aheu{} and the \alsh \ algorithm capture the key aspects that we hoped for: better performance in terms of computation time due to fast solving of MIPS instances, while working with an arbitrary collection of feasible assortments, and without much impact on revenue.}

\noindent\textit{Datasets}: \revision{We use real and  semi-synthetic data sets. The real dataset~\citep{taFeng} contains transaction data from a Chinese grocery story between November, 2000 and February, 2001. The dataset contains information about the bundles of products purchased in each transaction along with the prices of the products. We use \textsc{FPgrowth}~\citep{fournier2014spmf} to generate frequent itemsets with appropriate minimum supports and then prune ones with low cardinality (e.g., singletons) to obtain our collection of relevant assortments. We also use the transaction data to estimate an MNL model. The no-purchase utility cannot be estimated with this data, and is defined such that the probability of no purchase is 30\% when all the products are displayed.}
	
For generating additional problem instances in the general setting, we use four publicly available transaction datasets, viz., retail, foodmart, chainstore and e-commerce transaction logs~\citep{fournier2014spmf} to create collections of general assortments. Similar to the Ta Feng dataset, we generate frequent itemsets with appropriate minimum supports using \textsc{FPgrowth}. We augment these by generating synthetic prices and utilities that are inversely correlated. Table~\ref{table:freq-itemset-stat} describes some statistics of the assortments generated. Plots for experiments using additional semi-synthetic data (e.g., with prices independent of utilities with the aforementioned transaction datasets, as well as instances generated using the Billion prices dataset) are provided in \ref{sec:app_figs}.

\begin{table}[h]
	\begin{center}
		\resizebox{\columnwidth}{!}{%
			\begin{tabular}{ lccccc } 
				\toprule
				Dataset & Ta Feng & Retail & Foodmart & Chainstore & E-commerce \\
				\midrule 
				Number of transactions & 119390 & 88162 &4141 & 1112949 & 540455 \\ 
				Number of items & 23778
				& 3160 & 1559 & 321 & 2208 \\ 
				Number of general assortments & 71722  & 80524 & 81274 & 75853 & 23276 \\ 
				Size of largest assortment &  16 & 12 & 14 & 16 & 8\\ 
				Size of smallest assortment & 8 & 3 & 4 & 5 & 3\\ 
				\bottomrule
			\end{tabular}
		}%
	\caption{Assortments generated from transactional datasets. 
	\label{table:freq-itemset-stat}}
	\end{center}
\end{table}
\vspace{-.5cm}

\subsection{Experiments with General Assortments}

\revision{The first experiment explores how the algorithms fare over the full sized real and semi-synthetic instances. The performance of the algorithms are plotted in Figure~\ref{fig:gen-ast-real-inv} using both the real dataset and four semi-synthetic datasets detailed above. As can be inferred, all three proposed algorithms are $2$-$3\times$ faster than exhaustive search. In particular, even \ann \ which has higher time complexity than exhaustive search is faster. \aheu \ and \alsh \ are faster than \ann \ while making small sacrifices in the revenue. We observe that \alsh \ has better performance in terms of revenue than \aheu, while being comparable in computation time. }

\begin{figure}[ht]
	\centering
	\includegraphics[width=.48\textwidth]{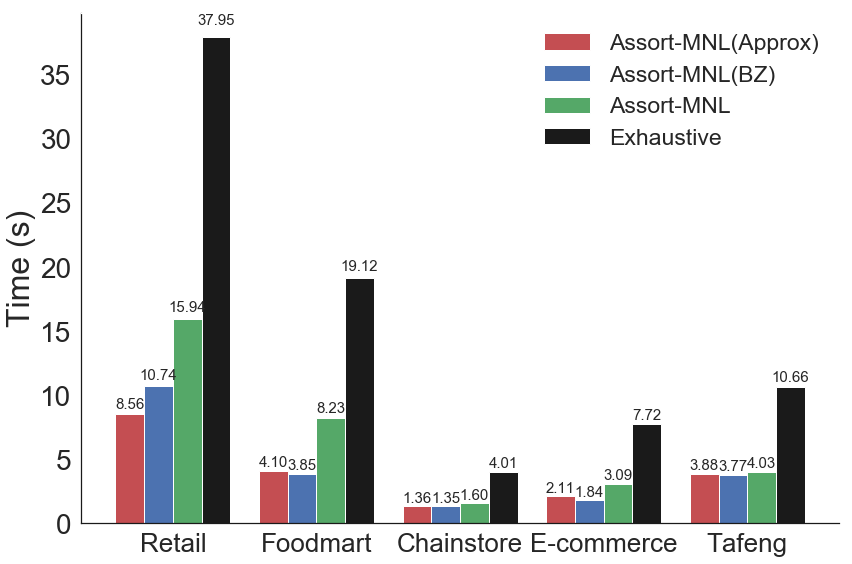}
	\includegraphics[width=.48\textwidth]{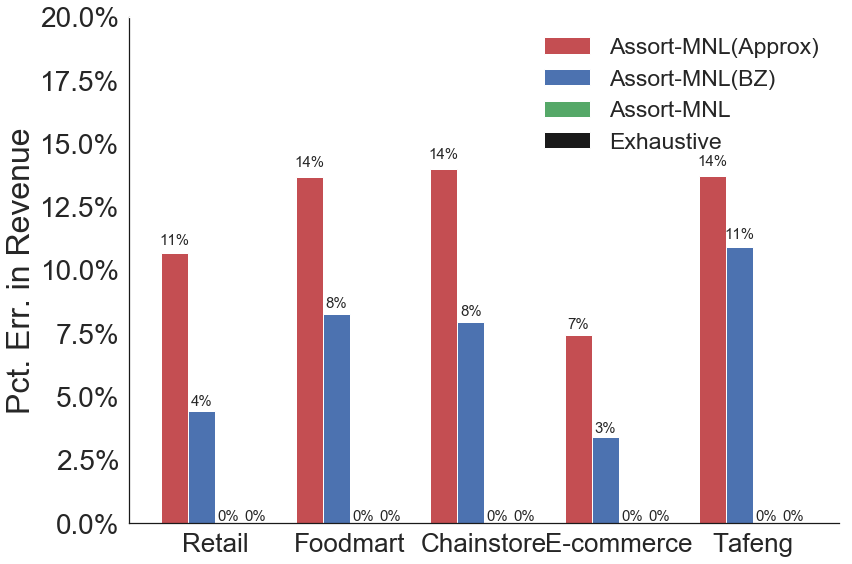}
	\caption{Performance of algorithms over general assortment instances. Apart from the Ta Feng dataset, we explicitly generate negatively correlated prices and utilities for the remaining.  \label{fig:gen-ast-real-inv}}
\end{figure}

\revision{ In the second experiment, we focus on how the performance of the algorithms varies as the number of assortments in the feasible set increases.  We focus on the Ta Feng dataset and vary the size of the set of feasible assortments in the range $\{100,200,400,800,1600,3200,6400, \allowbreak 12800, 25600, 51200\}$. Here, the set of feasible assortments is picked uniformly at random from the set of frequent itemsets.}
\revision{ We observe in Figure~\ref{fig:tafeng-gen-ast-real-price} that throughout the considered range of sizes of feasible assortments, our algorithms are much faster than exhaustive search. The difference becomes more pronounced as the size increases. In terms of error in revenue, both \aheu \ and \alsh \ have similar trends, and the revenue errors are consistently average around $12$\%.}
 
\begin{figure}[ht]
	\centering
	\includegraphics[width=.48\textwidth]{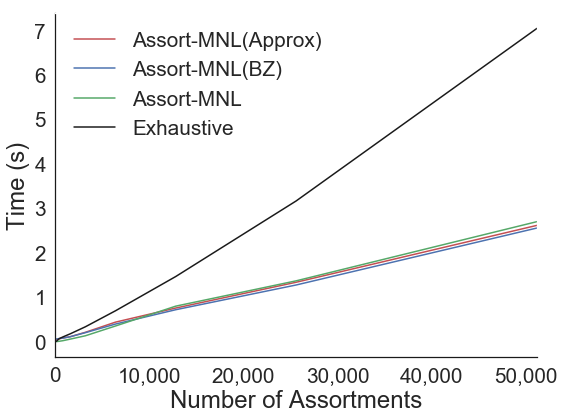}
	\includegraphics[width=.48\textwidth]{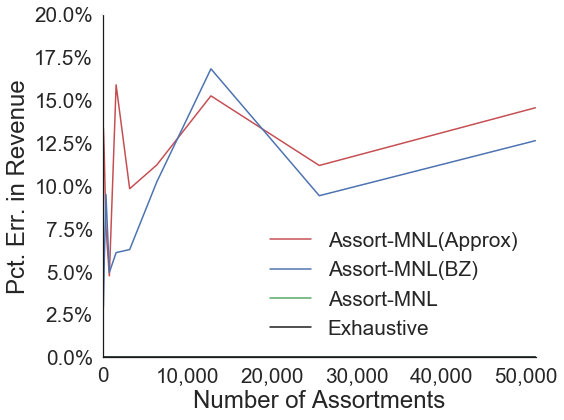}
	\caption{Performance of algorithms over general data-driven instances derived from the Ta Feng dataset. The x-axis corresponds to the number of feasible/relevant assortments.
		\label{fig:tafeng-gen-ast-real-price}}
\end{figure} 

From the above findings, we see the promise of our algorithms for \emph{near-real-time} web-scale assortment planning applications when the feasible region is not compactly represented (for additional results on similar lines, see \ref{sec:app_figs}).

\subsection{Experiments with Cardinality-constrained Assortments}

\begin{figure}[ht]
	\centering
	\includegraphics[width=.48\textwidth]{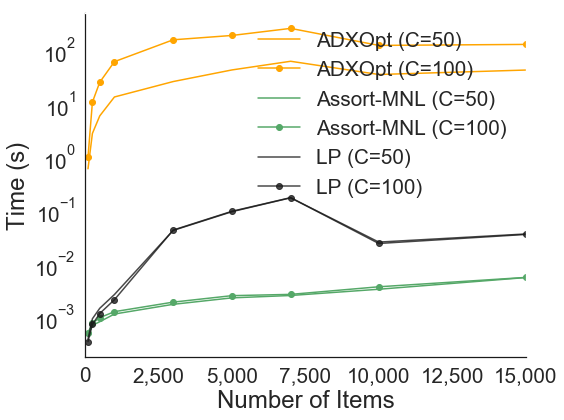}
	\includegraphics[width=.48\textwidth]{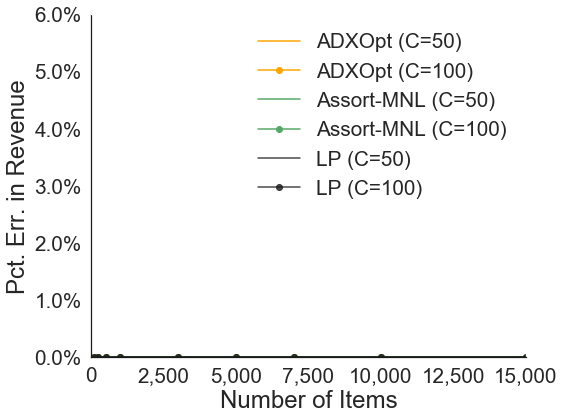}
	\caption{Performance of algorithms in the capacitated setting over instances derived from the Ta Feng dataset. The x-axis corresponds to the number of items.
		\label{fig:tafeng-cap-real-price-prod}}
\end{figure} 

\begin{figure}[ht]
	\centering
	\includegraphics[width=.32\textwidth]{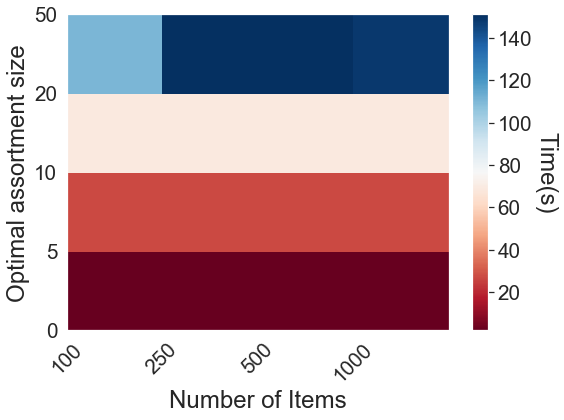}
	\includegraphics[width=.32\textwidth]{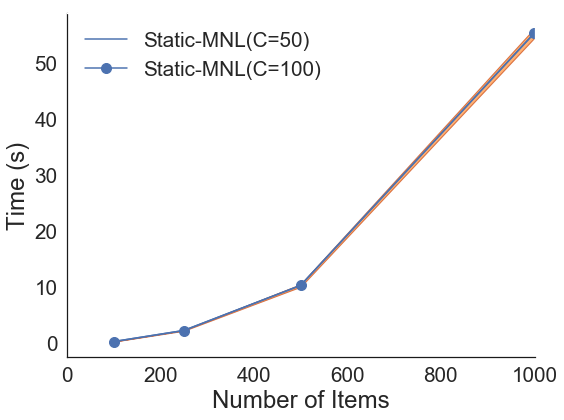}
	\includegraphics[width=.32\textwidth]{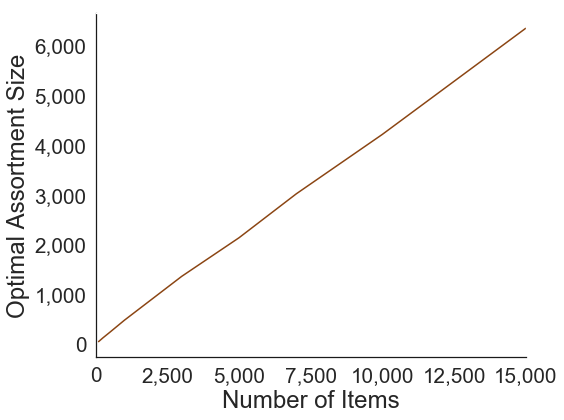}
	\caption{Performance of ADXOpt (left) and Static-MNL (center) when the cardinality constraint is $100$. The intensity (seconds) shows that the performance of ADXOpt is dependent on the size of the optimal assortment. \textsc{Static-MNL}'s timing results are an order of magnitude worse even for moderate-sized instances. The right plot shows the mean size of the optimal \emph{unconstrained} assortment. Instances for these results are generated using the Ta Feng dataset. 
		\label{fig:cap-real-price-prod-adxopt-static-mnl}}
\end{figure}

We explore how \ann \ fares when compared to specialized algorithms viz., ADXOpt, LP and \textsc{Static-MNL} in the capacity-constrained setting (exhaustive search is not considered because its performance is an order of magnitude worse than all these methods) \revision{using the Ta Feng dataset (additional results using the Billion prices dataset are in \ref{sec:app_figs}). We consider two values of the cardinality constraint - $50$ and $100$.  The number of items was varied in the range $\{100, 250, 500, 1000, 3000, 5000, 7000,10000,\allowbreak 15000\}$. }
 
\revision{Figure~\ref{fig:tafeng-cap-real-price-prod} shows the performance of all the algorithms considered except \textsc{Static-MNL}. First, we observe that for both values of the  cardinality constraints,  \ann \ (the only approximate algorithm) has revenue error very close to 0. In terms of runtime,  \ann \ runs much faster than both \textsc{ADXOpt} and LP. For instance, when the number of items is $15000$, \ann \ computes solutions in time that is almost \emph{four orders of magnitude} less compared to ADXOpt and \emph{one order of magnitude} less compared to LP on average (we discount the time needed to set up the LP instance and only report the time to solve the instance otherwise the relative gains would be much higher). Further, by making a comparison across two different values of cardinality constraint, we observe that the runtime does not vary much with the cardinality for \ann \ and the LP based solution. But owing its greedy search nature, the runtime of the ADXOpt algorithm increases as value of the cardinality constraint is increased.}

We have re-plotted the timing performance of ADXOpt on the Ta Feng dataset separately in Figure~\ref{fig:cap-real-price-prod-adxopt-static-mnl}(left), as its running time (plotted as intensity) varies as a function of the size of the optimal assortment, which is not the case with \ann \ and LP. If the instance happens to have a small optimal assortment, ADXOpt can get to this solution very quickly. On the other hand, it spends a lot of time when the optimal assortment size is large. This is illustrated through an intensity plot as opposed to a one-dimensional timing plot.  
\revision{We plot the timing performance of \textsc{Static-MNL} separately in Figure~\ref{fig:cap-real-price-prod-adxopt-static-mnl}(center). Its performance is close to ADXOpt and much worse than the \ann \ and LP methods. Further, owing to the stringent memory requirements of the algorithm, we had to limit the size of the problem instances to number of items less than $1000$. }

\revision{Further, we ensure that the cardinality constraint (for both values - $50$ and $100$) is not a vacuous constraint by obtaining the size of the optimal \emph{unconstrained} assortment, see Figure \ref{fig:cap-real-price-prod-adxopt-static-mnl}(right). We observe that the size of the optimal unconstrained assortment is larger than our chosen constraints on average.}
 
In summary, \ann \ is competitive with the state of the art algorithms, viz., ADXOpt, \textsc{Static-MNL} and LP in the capacitated setting. Further, \ann,  \aheu \ and \alsh \ are highly scalable in the general feasible assortment setting. All our algorithms provide efficient solutions to vastly larger assortment planning instances under the MNL choice model. For instance, we show computational results when instances have a number of items that is of the order of $\sim10^5$ easily, whereas the regime in which experiments of the current state of the art methods were carried out is less than $\noofprod \sim 10^3$, thus representing two orders of magnitude improvement. See additional experimental results providing similar conclusions in \ref{sec:app_figs}. 
\revision{Overall, these algorithms provide a way to efficiently trade-off accuracy with time and make way for a significant step towards meeting the ambitious time requirements on computation in online/digital platforms today.}
\section{Concluding Remarks}
\label{sec:conclude}

We proposed multiple efficient algorithms that solve the assortment optimization problem under the multinomial logit (MNL) purchase model even when the feasible assortments cannot be compactly represented. In particular, we motivate how frequent itemsets and business rules can be used to define candidate assortments, making the planning problem data-driven. Our algorithms are iterative and build on binary search and fast methods for maximum inner product search (MIPS) to find near optimal solutions for large scale instances. Though solving large scale instances efficiently under flexible assortments is a significant gain, there is scope for extending this work in many ways. For instance, studies by psychologists have revealed that buyers are affected by the assortment size as well as how frequently they change over the course of their interactions~\citep{iyengar2000choice}. Extending our methods to multi-purchase models is also a promising direction.

Our algorithms are meaningful both in the online and offline setups: in online setups such as e-commerce applications, our algorithms scale as the problem instances grow. In offline setups, our algorithms afford flexibility to the decision maker  by allowing optimization over arbitrary assortments, which could be driven by business insights or transaction log mining.

\section*{Funding}
This research did not receive any specific grant from funding agencies in the public, commercial, or not-for-profit sectors.

\bibliography{assort_lsh}

\newpage
\appendix
\centerline{\textbf{{\large SUPPLEMENTARY}}}
\vspace{2in}
\revision{
In this supplementary, we provide the following additional details, including proofs of claims made:
\begin{itemize}
	\item \ref{sec:additional-data-structures} discusses the details of a specific data structure, namely locality sensitive hashing (LSH), and how it enables solving \nmips{} efficiently.
	\item \ref{app:algo} provides proofs for claims made in Section~\ref{sec:optimize}, as well as provides some additional details on the design of the three algorithms.
	\item \ref{sec:app_figs} provides additional computation evidence that shows how the proposed methods compared against benchmarks in both the general setting as well as the capacitated setting using multiple semi-synthetic setups.
\end{itemize}
}
\newpage

\section{Additional Related Work}\label{sec:additional-related-work}

Here we discuss some additional references that are related to this work.

\noindent\textit{Capacitated Setting}: While no general exact methods other than integer programming exist when the feasibility constraints are not unimodular, the special case of capacity constrained setting has been well studied algorithmically in prior work. The key approaches in this special case are: (a) a linear programming (LP) based approach by ~\cite{sumida2020revenue} (time complexity O($\noofprod^{3.5}$) if based on interior point methods), (b) the \textsc{Static-MNL} algorithm  (time complexity O($\noofprod^2\log\noofprod$)) by ~\cite{Paat}, and (c) ADXOpt (time complexity $O(n^2bC)$, where $n$ is the number of items, $C$ is the maximum size of the assortment and $b = \min \{C, n-C+1 \}$ for the MNL choice model) by ~\cite{jagabathula2014assortment}. Note that ADXOpt is a local search heuristic designed to work with many different choice models, but is not guaranteed to find the optimal when a general collection of feasible assortments are considered under the MNL model. 

\noindent\textit{Other Choice Models}: There are many choice models that have been proposed in the literature~\citep{train2009discrete} including Multinomial Logit (MNL), the mixture of MNLs (MMNL) model, the nested logit model, multinomial probit, the distribution over rankings model~\citep{farias2013nonparametric}, exponomial ~\citep{alptekinouglu2016exponomial} choice model and the Markov chain model~\cite{desir2015capacity}. Though some of the above choice models are more expressive than the MNL choice model, they are harder to estimate reliably from data as well as lead to hard assortment optimization problems.  Offline methods that work with such behavior models include a mixed integer programming approach in ~\cite{Vel} (NP-hard) for the distribution over ranking model, ~\cite{davis2014assortment} for the nested logit, ~\cite{rusmevichientong2014assortment} for the mixture of MNLs model, and ~\cite{desir2015capacity} for the Markov chain model, to name a few. Capacity and other business constraints have also been considered for some of these  choice models (see~\cite{sumida2020revenue}). It is not clear how to scale these algorithmic approaches or extend these methods when the feasible assortments cannot be compactly represented.

\noindent\textit{Online Assortment Planning}: Finally, note that although one could look at continuous estimation and optimization of assortments in an online learning framework~\citep{shipra}, decoupling and solving these two operations separately brings in a lot of flexibility for both steps, especially because the optimization approaches can be much better tuned for performance, which is our focus. Even when one is interested in an online learning scheme, efficient algorithms (such as the ones we propose) can be easily used as drop-in replacements for subroutines in regret minimizing online assortment optimization schemes.

\noindent\textit{Empirical Success of the MNL Model}:  The MNL model is arguably one of the most commonly used choice model, due to its simplicity.  In a recent work, viz., ~\cite{feldman2018taking} that succeeds and complements this work, the authors comprehensively demonstrate the significance of using MNL models to optimize and show assortments on Alibaba's Tmall/Taobao e-commerce platforms using field experiments. They show that although the MNL based model has lower predictive power compared to a richer machine learning based choice behavior model, it is able to generate significantly higher revenue ($\sim 28\%$) compared to the latter due to explicitly modeling substitution behavior. They also show how being able to compute optimal assortments for the MNL in real-time is also a significant contributing factor in their deployment. Another paper that points out to advantages of discrete choice models over machine learning models for predicting choice behaviour is \cite{raval2019machine}. They empirically observe that traditional econometric discrete choice models perform much better when there has been a major change in the choice environment.

\noindent\textit{Fast Solution to MIPS and Similarity Search}: \nmips \ problems can be solved using methods based on Locality Sensitive Hashing (LSH)~\citep{neyshabur2015symmetric} and variants such as \textsc{L2-ALSH(SL)}~\citep{shrivastava2014asymmetric} and \textsc{Simple-LSH}~\citep{neyshabur2015symmetric}. Approximate methods have also been proposed for related problems such as the Jaccard Similarity (JS) search in the information retrieval literature. For instance, one can find a set maximizing Jaccard Similarity with a query set using a technique called \textsc{Minhash} \citep{broder1997resemblance}. The \cnn \ problem can be addressed using \textsc{L2LSH}~\citep{datar2004locality} and many other solution approaches (see for instance~\cite{li17d} and references within). 
\section{Data Structures for Solving \nmips}\label{sec:additional-data-structures}

In this section, we provide additional details about the underlying data structures that allow for solving the \nmips{} and \cnn{} problems in an efficient manner. In particular, we will focus on locality sensitive hashing.

\subsection{Locality Sensitive Hashing (LSH)}\label{subsec:LSH}

LSH~\citep{andoni2008near,har2012approximate} is a technique for finding vectors from a known set of vectors (referred to as the search space), that are  `similar' (i.e, \emph{neighbors} according to some metric) to  a given query vector in an efficient manner. It uses hash functions that produce similar values for input points (vectors) which are  similar to each other as compared to points which are not.  

We first formally define the closely related (but distinct) problems of finding the \textit{near neighbor} and the \textit{nearest neighbor} of a point from a set of points denoted by $P$ (with $|P| = N$).

\begin{definition} The $(c,r)$-NN (approximate \emph{near} neighbor) problem with failure probability $f \in (0,1)$ is to construct a data structure over a set of points $P$ that supports the following query:  given point $q$, if $\min_{p \in P}d(q,p) \leq r$, then report some point $p' \in P\cap \{p: d(p,q) \leq cr\}$ with probability $1-f$. Here, $d(q,p)$ represents the distance between points $q$ and $p$ according to a metric that captures the notion of neighbors. Similarly, the $c$-NN (approximate \emph{nearest} neighbor) problem with failure probability $f \in (0,1)$ is to construct a data structure over a set of points $P$ that supports the following query: given point $q$, report a $c$-approximate nearest neighbor of $q$ in $P$ i.e., return $p'$ such that $d(p',q) \leq c\min_{p \in P}d(p,q)$ with probability $1-f$.
\end{definition}

As mentioned earlier, the data structures that can solve the above search problems efficiently are based on hash functions, which are typically endowed with the following properties:

\begin{definition}\label{def:hashfunctions}
A $(r,cr,P_1,P_2)$-sensitive family of hash functions ($h \in \mathcal{H}$) for a metric space $(X,d)$ satisfies the following properties for any two points $p,q \in X$ and $h$ chosen uniformly at random from  $\mathcal{H}$:
\begin{itemize}
\item If $d(p,q) \leq r$, then $Pr_{\mathcal{H}}[h(q) = h(p)] \geq P_1$, and
\item If $d(p,q) \geq cr$, then $Pr_{\mathcal{H}}[h(q) = h(p)] \leq P_2.$
\end{itemize}
\end{definition}

In other words, if we choose a hash function uniformly from $\mathcal{H}$, then they will evaluate to the same value for a pair of points $p$ and $q$ with high ($P_1$) or low ($P_2$) probability based on how similar they are to each other. The following theorem states that we can construct a data structure that solves the approximate near neighbor problem in sub-linear time.

\begin{theorem} [\cite{har2012approximate} Theorem 3.4] Given a $(r,cr,P_1,P_2)$-sensitive family of hash functions, there exists a data structure for the $(c,r)$-NN (approximate near neighbor problem) with the Euclidean norm as the distance metric, over $n$ dimensional points in the set $P$ (with $|P| = N$) such that the time complexity of returning a result is $O(nN^\rho/P_1 \log_{1/P_2}N)$ and the space complexity is $O(nN^{1+\rho}/P_1)$ with $\rho = \frac{\log 1/P_1}{\log 1/P_2}$. Further, the failure probability is upper bounded by $1/3 + 1/e$.\label{thm:near-neighbor}
\end{theorem}

\begin{remark} Note that the failure probability can be changed to meet to any desired upper bound by \emph{amplification}: here we query a constant number (say $\kappa$) of different copies of the data structure above and aggregate the results to output a single near-neighbor candidate. Such an output will fail to be a $(c,r)$-NN with probability upper bounded by $(1/3+1/e)^{\kappa}$. 
\end{remark}

While the hash family with the properties stated in Definition~\ref{def:hashfunctions} already gives us a potential solution for the $(c,r)$-NN problem, the following data structure allows for a finer control of the time complexity and the approximation guarantees (corresponding to Theorem~\ref{thm:near-neighbor}). In particular, as described in~\cite{andoni2008near}, we employ multiple hash functions to increase the confidence in reporting near neighbors by amplifying the gap between $P_1$ and $P_2$. The number of such hash functions is determined by suitably chosen multiplicity parameters $L_1$ and $L_2$. In particular, we can choose $L_2$ functions of dimension $L_1$, denoted as $g_j(q) = ( h_{1,j}(q) , h_{2,j}(q), \cdots h_{L_1, j}(q) )$, where $h_{t,j}$ with $1 \leq t \leq L_1 ,1 \leq j \leq L_2$ are chosen independently and uniformly at random from the family of hash functions. The data structure for searching points with high similarity is constructed by taking each point $x$ in the search space and storing it in the location (bucket) indexed by  $g_j(x) , 1 \leq j \leq L_2$.  When a new query point $q$ is received,  $g_j(q) , 1 \leq j \leq L_2$ are  calculated and all the points from the search space in the buckets  $g_j(q) , 1 \leq j \leq L_2$ are retrieved.  We then compute the similarity of these points with  the query vector in a sequential manner and return any point that has a similarity greater than the specified threshold $r$. We also interrupt the search after finding the first $L_3$ points including duplicates for a suitably chosen value of $L_3$ (this is necessary for the guarantees in Theorem~\ref{thm:near-neighbor} to hold).

Given $P_1$, the probability that any point  $p$ such that  $d(q,p) \leq cr$ is retrieved is at least $ 1 - (1 - P_1^{L_1})^{L_2} $. Thus for any desired error probability $\delta > 0$, we can choose  $L_1$ and $L_2$ such that $p$ is returned by the data structure with probability at least $1 - \delta$. The time taken to perform a query, i.e., retrieve a point with the highest similarity can be as low as $O(nN^\rho)$, where $\rho = \frac{\ln (1/P_1)}{\ln (1/P_2)}$ is a number less than $1$, $N$ is the number of points in the search space and $n$ is the dimensionality of each point (see Theorem~\ref{thm:near-neighbor} above).
The key insight that allows for this sublinear dependence on the number of points is due to the appropriate choices of $L_1$, $L_2$ and $L_3$, which determine collisions of hashed (i.e., projected) transformations of the original points. In particular, choosing $L_1 = \log N$, $L_2 = N^\rho$ and $L_3 = 3L_2$ gives the desired computational performance (in addition to influencing the error probabilities). This sublinearity of retrieval time is what drives the computational efficiency of hashing based retrieval methods including our algorithms (in Section \ref{sec:optimize}).  
Finally, note that the memory requirements of the data structure turns out to be $O(nN^{1+\rho})$ (as stated in Theorem~\ref{thm:near-neighbor}), which is not too expensive in many large scale settings (we need $O(nN)$ space just to store the points).

It turns out that out the problem we care about in this work is related to the \emph{nearest} neighbor problem. And the nearest neighbor problem has a close connection to the \emph{near} neighbor problem described above. In particular, the following theorem states that an approximate nearest neighbor data structure can be constructed using an approximate near neighbor subroutine without taking a large hit in the time complexity, the space complexity or the retrieval probability.

\begin{theorem} [\cite{har2012approximate} Theorem 2.9] Let P be a given set of N points in a metric space, and let $c, f \in (0,1)$ and $\gamma \in (\frac{1}{N},1)$ be parameters. Assume we have a data-structure for the $(c,r)$-NN (approximate \emph{near} neighbor) problem that uses space $S$ and has query time $Q$ and failure probability $f$. Then there exists a data structure for answering $c(1+O(\gamma))$-NN (approximate \emph{nearest} neighbor) problem queries in time $O(Q\log N)$ with failure probability $O(f\log N)$. The resulting data structure uses $O(S/\gamma \log^2 N)$ space.\label{thm:nearest-neighbor}
\end{theorem}

To give intuition about how near neighbor data structures can be used to compute the nearest neighbor, we describe a simple strategy, which is as follows. We create multiple near neighbor data structures as described in Theorem~\ref{thm:near-neighbor} using different threshold values ($r$) but with the same success probability $1-f$ (by amplification for instance). When a query vector is received, we calculate the near-neighbors using the hash structure with the lowest threshold. We continue checking with increasing value of thresholds till we find at least one near neighbor. Let $\widetilde{r}$ be the first threshold for which there is at least one near neighbor. This implies that the probability that we don't find the true nearest neighbor is at most $f$ because the near neighbor data structure with the threshold $\widetilde{r}$ has success probability $1-f$.  Assuming that the different radii in the data structures are such that the number of points returned for the threshold  $\widetilde{r}$ is sublinear in $N$ i.e. $O(N^\eta)$ for some $\eta < 1$, gives us the desired data structure for the nearest neighbor problem. While the last assumption is reasonable as long as the different radii are not very far apart, the structure corresponding to Theorem~\ref{thm:nearest-neighbor} does not need it. Further details on this data structure can be found in ~\cite{har2012approximate}.

Finally, we remark about the properties of the hash families characterized by parameters $\rho$ and $c$. These two parameters are inversely related to each other. For instance, in \cite{andoni2008near}, the authors propose a family $\mathcal{H}$ for which $\rho(c) = \frac{1}{c^2} + O(\log\log N / \log^{1/3}N)$. For large enough $N$, and for say $c = 2$ (2 approximate near neighbor), $\rho(c) \approx .25$, and thus the query completion time of near neighbor (as well as nearest neighbor) queries is $\propto N^{.25}$ which is a significant computational saving.

\subsection{Solving \nmips{} and \cnn{} using LSH}

We illustrate how the MIPS problem can be solved approximately using a specific LSH family, as described in~\cite{neyshabur2015symmetric}. For $x \in \mathcal{R}^{\noofprod}, || x ||_2 \leq 1$, we first define a preprocessing transformation $T: \mathcal{R}^{\noofprod} \rightarrow \mathcal{R}^{\noofprod+1}$ as $ T(x) =  [x; \sqrt{1- ||x||_2^2}]$. We sample a spherical random vector $a \sim \mathcal{N}(0,I)$ and define the hash function as $h_a(x) = sign(a\cdot x)$. All points in the search space are preprocessed as per the transformation $T(\cdot)$ and then hashed using hash functions, such as the above, to get a set of indexed points.
During run-time when we obtain the query vector, it is also processed in the same way i.e. first transformed through $T(\cdot)$ and then 
through the same hash functions that were used before. Checking for collisions in a way that is similar to the near neighbor problem in Section~\ref{subsec:LSH} retrieves the desired similar points.

Without loss of generality, assume that the query vector $y$ has $||y||_2 = 1$. The following guarantee can be shown for the probability of collision of two hashes:
$$\mathcal{P}[h_a(T(x)) = h_a(T(y))] = 1 - \frac{cos^{-1}(x\cdot y)}{\pi},$$ 
which is a  decreasing function of the inner product $x\cdot y$. For any chosen threshold value $D$ and  $c<1$, we consequently get the following:
\begin{itemize}
\item  If  $x\cdot y \geq D$, then
$\mathcal{P}[h_a(T(x)) = h_a(T(y))]  \geq 1 - \frac{cos^{-1}(D)}{\pi}$.
\item  If  $x\cdot y \leq cD$, then
$\mathcal{P}[h_a(T(x)) = h_a(T(y))]  \leq 1 - \frac{cos^{-1}(cD)}{\pi}$.
\end{itemize}

We are by no means restricted to using the above hash family. Because of the transformation $T(\cdot)$, we have obtained a nearest neighbor problem that is equivalent to the original inner product search problem. Thus, any hash family that is appropriate for the Euclidean nearest neighbor problem can be used (see~\ref{subsec:LSH}). 

\section{Additional Information on Proposed Algorithms} \label{app:algo}

In this section, we provide: (a) additional examples of business rules that lead to capacity constraints and how to handle them (\ref{app:capacity-extensions}), (b) justification for the updates in \aheu \ (\ref{subsec:aheu-update}), (c) modification of \aheu for experiments (\ref{subsec:modified-aheu}), and (d) proofs of correctness as well as time-complexities of algorithms (\ref{app:algo_correct}).

\subsection{Other Examples of Capacity Constraints}\label{app:capacity-extensions}

\noindent(1) \textit{Lower Bound on Assortment Size} - This constraint requires that all feasible assortments have at least $c \leq C$ items in addition to having at most $C$ items. In every iteration of Algorithm \ref{alg:ann_outline}, we find the top-$C$ items (by the  product $v_i(p_i - K)$)  that give a positive value of the inner product. This can be modified to finding the top-$C$ items and including at least the top-$c$ of them irrespective of the sign of the product. \\

\noindent(2) \textit{Capacity Constraints on Subsets of Items} - To ensure diversity in the assortment, we may have constraints on how many items can be chosen from certain subsets of items. Suppose the items are partitioned into subsets $B_1, B_2, \cdots, B_w$ and we have capacity constraints $C_1, C_2, \cdots, C_w$ on each subset respectively, then the comparison to be solved in every iteration can be written as: 
	\begin{eqnarray*}
		K \leq   \max_{S: |S| \leq C_1, S \subseteq B_1} \frac{1}{v_0}   \sum_{i \in S}  v_i (p_i - K)  + \cdots +   \max_{S: |S| \leq C_w, S \subseteq B_w} \frac{1}{v_0}  \sum_{i \in S}  v_i (p_i - K). 
	\end{eqnarray*} 
One can solve these $w$ independent problems the same way as the previous setting.\\

\noindent(3) \textit{Assortments Near a Preferred Reference Assortment} - This constraint requires that our search should only consider those assortments that are near a preferred/status-quo assortment (for instance, the currently deployed one). In the extreme case where these preferred items must all be in the final assortment, the problem reduces to an optimization for the remaining wiggle room in capacity. Otherwise, we can impose a constraint that at least some of the items from the preferred assortment are in the final solution (see (1) above), and simultaneously impose a reduced capacity constraint on the remaining items (see (2) above).

\subsection{Justifying the Updates in \aheu}\label{subsec:aheu-update}

Here we provide a formal proof for one of the interval updates in \aheu. The other two updates can be proved along similar lines.

\revision{
\begin{proposition}
If $\hat{K} \geq \approxsol \cdot \bv$, then $K \geq\exactsol \cdot \bv$.
\end{proposition}
\begin{proof} 
We will prove the contra-positive of this statement. That is, we will show that if $\exactsol$ is such that $\exactsol\cdot \bv$ is indeed greater than $K$, then $\hat{K} \leq \approxsol \cdot \bv$.  If we could solve the exact MIPS problem with the query vector $\bv$ (that depends on $K$), we would get the solution $\exactsol$ and immediately conclude that $\exactsol\cdot \bv \geq K$. With the $\apargmax$ operation (with failure probability $f=0$), we are returned an approximate solution $\approxsol$ with the guarantee that $1 + (1+\nu)^2(\exactsol\cdot \bv- 1) \leq \approxsol \cdot \bv\leq \exactsol\cdot \bv$. But since, $K \leq \exactsol\cdot \bv$, we also have $1 + (1+\nu)^2(K - 1) \leq 1 + (1+\nu)^2(\exactsol\cdot \bv - 1)$. This readily implies $\hat{K} \leq \approxsol \cdot \bv$. 
\end{proof}
}

\subsection{Modified \aheu{} for Experiments}\label{subsec:modified-aheu}

In Algorithm~\ref{alg:aheu_simpler}, we simply perform the approximate comparison (seen in \aheu) within the binary search loop specified by \ann. In other words, even though the comparison is approximate, we ignore this while updating the interval. The reason we do this, even though there is a change of not narrowing down into an appropriate interval, is for the experiments. The underlying data structures don't expose a fixed $\nu$, which is needed as a parameter in \aheu, and instead optimize it in a data-driven way (see Section~\ref{sec:experiments}). And because we don't have a direct handle on this parameter, we modify \aheu into the simpler heuristic in Algorithm~\ref{alg:aheu_simpler} (which does not depend on $\nu$).

\begin{algorithm}[H]
  \caption{\small{Simpler} \aheu{} for Experiments}
  \label{alg:aheu_simpler}
  \begin{algorithmic} 
    \REQUIRE{ Prices $\{p_i\}_{i=1}^{n}$, tolerance parameter $\epsilon$} \\
    \STATE{$L_1 = 0, \ U_1 = p_1 , \ t=1,\bp = (p_1, \cdots, p_n)$} 
    \STATE{$\bu^S = (u_1, u_2, \cdots u_n) \ \text{ where } \  u_i = \mathbf{1} \{ i \in S\} , \text{ for any } S \in \cS $}
    \STATE{$\hat{S} = \{1 \}, \ \mathbf{\widehat{Z}} = \{ \mathbf{\hat{z}}^S | \hat{\bz}^S = \left( \bp \circ \bu^S, \bu^S \right), S \in \cS \}$}
    \WHILE{ $U_t - L_t > \epsilon$} 
    \STATE{ $K = \frac{L_t + U_t}{2}$} \\
    \STATE{$\mathbf{\hat{v}_K} = (v_1, \cdots, v_n, -v_1K,-v_2K, \cdots -v_nK)$} \\
    \STATE{$ \hat{\bz}^{\tilde{S}_j} = \apargmax_{\mathbf{\hat{z}}^S \in \mathbf{\hat{Z}}} \ \mathbf{\hat{v}_K} \cdot \mathbf{\hat{z}}^S  $}
    \IF{$K \leq  \frac{ \bv \cdot \hat{\bz}^{\tilde{S}}}{v_0}    $}  
    \STATE{$L_{t+1} =  K, U_{t+1} = U_t$, $\hat{S}  = \tilde{S}_j $}
    \ELSE 
    \STATE{$L_{t+1} = L_t , U_{t+1} = K$} \\
    \ENDIF 
    \ENDWHILE
    \RETURN{$ \hat{S} $}
  \end{algorithmic}
\end{algorithm}

\subsection{Proofs for Correctness and Time-Complexities} \label{app:algo_correct}
In this section, we provides proofs for the following list of claims:
\begin{itemize}
  \item Lemma \ref{lem:ann_correct}
  \item Lemma~\ref{lem:ann_capacity_time}
  \item Lemma~\ref{lemma:ann-general-time-complexity}
  \item Lemma \ref{lem:aheu_correct}
  \item Lemma~\ref{lemma:aheu-iteration-complexity}
  \item Theorem \ref{thm:BZError}
  \item Lemma~\ref{lemma:alsh-time-complexity}
\end{itemize}

\begin{proof}[Proof of Lemma \ref{lem:ann_correct}]
	
\revision{	In every iteration of the binary search procedure in \ann, we cut down the search space for the optimal revenue by half. We start with a search space of range $[0,p_1]$. Thus, the number of iterations to get to the desired tolerance is $ T = \left\lceil \log \frac{p_1}{\epsilon} \right \rceil$.}
	
 Let $\widetilde{T}$ be the last iteration $j \leq T $ such that $K_j \leq \max_{S \in \mathcal{S}} \revsv$. Hence, $ \revshatv \geq K_{\widetilde{T}}$.
	\noindent By the update rule, 
	$ L_{\widetilde{T} + 1} $ = $K_{\widetilde{T}}$ and  $ L_{j} = L_{\widetilde{T} +1}  =   K_{\widetilde{T}} \ \forall \  \widetilde{T}+1 < j  \leq T$.
	Hence, $ L_T \leq  \revshatv \leq U_T$.
	
	Using the termination condition, we know that $U_T - L_T \leq \epsilon$ and $ L_T \leq \revstarv \leq U_T $. Thus, $\revshatv \geq \revstarv -  \epsilon$.
	
\end{proof}

\revision{
\begin{proof}[Proof of Lemma~\ref{lem:ann_capacity_time}] 
  When using \ann \ under capacity constraint, as no special data structures are needed, the only space requirement is that of storing the price and utility parameters leading to a space complexity of $O(n)$. \\
  \indent The  \textsc{compare-step} in \ann \ amounts to finding top $C$ elements among $n$ elements. This has a complexity of $O(n\log C)$.   From Lemma \ref{lem:ann_correct}, the number of binary search iterations is  $ \left\lceil \log \frac{p_1}{\epsilon} \right \rceil$ leading to a time complexity of $O(n\log C \log \frac{p_1}{\epsilon})$.
\end{proof}
}

\revision{
 \begin{proof}[Proof of Lemma~\ref{lemma:ann-general-time-complexity}]
  In \ann,  the MIPS query in every iteration can be solved exactly in time $O(\noofprod \noofset)$ in a trivial manner when there is no compact representation for the set of feasible assortments. From Lemma \ref{lem:ann_correct}, the number of binary search iterations is  $ \left\lceil \log \frac{p_1}{\epsilon} \right \rceil$ leading to a run-time of $O(\noofprod \noofset \log \frac{p_1}{\epsilon}  )$. \\
  \indent As \ann \ doesn't require any additional storage except the space needed for storing all the feasible sets. Thus, the space complexity is $O(nN)$. 
 \end{proof}
 }

\begin{proof}[Proof of Lemma \ref{lem:aheu_correct}]
	
	Let the total number of iterations in the binary search be $T$ and let $\widetilde{T}$ be the last iteration $j \leq T $ such that $\hat{K}_j \leq { \bv \cdot \hat{\bz}^{\tilde{S}_j}}/ {v_0}  $. Hence, $ \revshatv \geq { \bv \cdot \hat{\bz}^{\tilde{S}_{\widetilde{T}}}}/ {v_0}  $.
	
	\noindent By the update rule, $ L_{j} = L_{\widetilde{T} +1}  \ \forall \  \widetilde{T}+1 < j  \leq T.$ 
	It is also easy to see that $\revshatv \geq L_{\widetilde{T} +1}.$ 
	Hence, $ L_T \leq  \revshatv \leq U_T.$
	\noindent Using the termination condition, we know that $U_T - L_T \leq \epsilon$ and $ L_T \leq \revstarv \leq U_T $. Thus, $\revshatv \geq \revstarv -  \epsilon$.
\end{proof}

\begin{proof}[Proof of Lemma~\ref{lemma:aheu-iteration-complexity}]
Let $I_j$ denote the size of the search interval in the $j$-th iteration of \aheu , i.e., $I_j = U_j - L_j$. As described there are three possible updates of the search interval. In the first two updates, $I_{j+1} = \frac{I_j}{2}$. In the third update rule, 
\begin{align*}
I_{j+1} &= U_j - \hat{K} = \frac{I_j}{2} + (\nu^2 + 2\nu)(1-K_j) \leq \frac{I_j}{2} + (\nu^2 + 2\nu).
\end{align*}

Define $\hat{\nu} = \nu^2 + 2\nu$. Then, in all the three cases $I_{j+1} \leq \frac{I_j}{2} + \hat{\nu}$.
Similarly, 
\begin{align*}
I_{j} \   \leq  \ \frac{I_{j-1}}{2} + \hat{\nu}, 
I_{j-1} \  \leq \ \frac{I_{j-2}}{2} + \hat{\nu}, 
\cdots
I_1 \  \leq \ \frac{I_{0}}{2} + \hat{\nu}.
\end{align*}

where $I_0 = p_1$. Taking a telescopic sum, we get
\begin{align*}
I_j \ & \leq \ \frac{I_{0}}{2^t} + \hat{\nu} \left( 1 + \frac{1}{2} + \cdots \frac{1}{2^{j-1}} \right), \\
& = \ \frac{I_{0}}{2^t} + 2\hat{\nu} \left( 1- \frac{1}{2^t} \right) \leq  \ \frac{I_{0}}{2^t} + 2\hat{\nu}.
\end{align*}

Thus, the number of iterations required to get the size of the search interval to within $\epsilon$ is $\left\lceil \log_2 \frac{p_1}{\epsilon -2\hat{\nu}}\right\rceil$.

\end{proof}

\begin{proof}[Proof of Theorem \ref{thm:BZError} ]

The proof from \cite{burnashev1974interval} (that quantifies the error in the vanilla BZ algorithm) has been modified for our setting in the following manner: (a) we remove the restriction that the noise distribution in the $\apargmax$ operation should be Bernoulli, and (b) we generalize the setting such that the error probability $p_j$ in the $\apargmax$ operation at every iteration $j$ can be different with $P_e = \max_{j \in {1, \cdots T}} p_j$.

The interval $[0,p_1]$ is divided into subintervals of width $\epsilon, \ a_i(j)$ denotes the posterior probability that the
optimal revenue $\theta^*$ is located in the $i$-th subinterval after the $j$-th iteration, and $\theta_j$ denotes the median of the posterior after  the $j$-th iteration. Let $Y_j = h(K_j)$ denote the outcome of the comparison. 
Let $\theta^*$ be fixed but arbitrary, and define $u(\theta^*)$ to be the index of the bin $I_i$ containing $\theta^*$, that is $\theta^* \in I_{u(\theta^*)}$. In general, let $u(j)$ be defined as the index of the bin containing the median of the posterior distribution in the $j$-th iteration.
We define two more functions of $\theta^*$, $M_{\theta^*}(j)$ and $N_{\theta^*}(j)$ as below - 

\begin{gather*}
M_{\theta^*}(j) = \frac{1- a_{u(\theta^*)}(j)}{a_{u(\theta^*)}(j)}, \text{ and} \\
N_{\theta^*}(j + 1) = \frac{M_{\theta^*}(j + 1)}{M_{\theta^*}(j)} =\frac{ a_{u(\theta^*)}(j)(1 -  a_{u(\theta^*)}(j + 1))} {a_{u(\theta^*)}(j + 1)(1 -  a_{u(\theta^*)}(j))}.
\end{gather*}

After $T$ observations our estimate of $\theta^*$ is the median of the posterior density $\pi_T(x)$, which means that $\theta_T \in I_{u(T)}$. Taking this into account we conclude that

\begin{align*}
P(|\theta_T - \theta^*| > \epsilon) \ & \leq \ P(a_{u(\theta^*)}(j) < 1/2), \\
& = \ P(M_{\theta^*}(T) > 1), \\
& \leq \ E[M_{\theta^*}(T)] \hspace{2mm} (\text{because of Markov's inequality}).
\end{align*}

Using the definition of $N_{\theta^*}(j)$, and manipulating conditional expectations we get:
\begin{align*}
E[M_{\theta^*}(T)] & =  E[M_{\theta^*}(T - 1)N_{\theta^*}(T)], \\
& = E [E[M_{\theta^*}(T - 1)N_{\theta^*}(T)|\ba(T - 1)]], \\
& = E [M_{\theta^*}(T - 1)E [N_{\theta^*}(T)|\ba(T - 1)]], \\
& \vdots \\
& = M_{\theta^*}(0)E [E[N_{\theta^*}(1)|\ba(0)] \cdots E[N_{\theta^*}(T)|\ba(T - 1)]], \\
&\leq M_{\theta^*}(0) \left\lbrace \max_{j \in \lbrace0,1,\cdots T-1 \rbrace}\max_{\ba(j)}E[N_{\theta^*}(j+1)|\ba(j)] \right\rbrace ^T.
\end{align*}

The error in the $\apargmax$ operation is not Bernoulli but has the following behavior. If $K_{j} > \theta^*$, then $Y_j= 0$ (there is no error). If $K_{j} < \theta^*$, then an error can occur and 
\begin{center}
$h(K_{j}) = 
\begin{cases}
1 & \text{with probability } 1-p_j, \textrm{ and} \\
0 & \text{with probability }p_j. \\
\end{cases}
$ \\
\end{center}

To bound $P(|\hat{\theta}_T - \theta^*| > \epsilon)$, we are going to consider three cases: (i) $u(j) = u(\theta^*)$; (ii) $u(j) > u(\theta^*)$; and (iii) $u(j) < u(\theta^*)$. For each of these cases, we first derive an expression for $N_{\theta^*}(j + 1)$ (with some algebraic manipulation and simplification) as follows:

$N_{\theta^*}(j+1) = 
\begin{cases}
\frac{1+(\beta-\alpha)x}{2\beta} \text{ with probability } B=1-A. \textrm{ and} \\ 
\frac{1-(\beta-\alpha)x}{2\alpha} \text{ with probability } A. \\ 
\end{cases}
$ \\
Thus:
\begin{enumerate}[(i)]
\item When $u(j) = u(\theta^*)$ and
  \begin{enumerate}
  \item $K_{j+1} = p_1^{-1}\epsilon (u(j)-1) \text{, then } 
   x = \frac{\tau_1(j) - a_{u(\theta^*)}(j)}{1- a_{u(\theta^*)}(j)}, \textrm{ and } A = p_{j+1}$.
   \item $K_{j+1} = p_1^{-1}\epsilon u(j) \text{, then } 
   x = \frac{\tau_2(j) - a_{u(\theta^*)}(j)}{1- a_{u(\theta^*)}(j)}, \textrm{ and }  A = 0$.
  \end{enumerate}
  \item When $u(j) > u(\theta^*)$ and
  \begin{enumerate}
  \item $K_{j+1} = p_1^{-1}\epsilon (u(j)-1) \text{, then } 
     x =  - \frac{\tau_1(j) + a_{u(\theta^*)}(j)}{1- a_{u(\theta^*)}(j)}, \textrm{ and }  A = 0$.
     \item $K_{j+1} = p_1^{-1}\epsilon u(j) \text{, then } 
     x = \frac{\tau_2(j) - a_{u(\theta^*)}(j)}{1- a_{u(\theta^*)}(j)}, \textrm{ and }  A = 0$.
  \end{enumerate}
   \item When $u(j) < u(\theta^*)$ and
  \begin{enumerate}
  \item $K_{j+1} = p_1^{-1}\epsilon (u(j)-1) \text{, then } 
       x =   \frac{\tau_1(j) - a_{u(\theta^*)}(j)}{1- a_{u(\theta^*)}(j)}, \textrm{ and }  A = p_{j+1}$.
       \item $K_{j+1} = p_1^{-1}\epsilon u(j) \text{, then } 
       x = - \frac{\tau_2(j) + a_{u(\theta^*)}(j)}{1- a_{u(\theta^*)}(j)}, \textrm{ and }  A = p_{j+1}$.
    \end{enumerate}
\end{enumerate}

As $0\leq \tau_1(j) \leq 1$ and $0 < \tau_2(j) \leq 1$, we have $|x| \leq 1$ in all the above cases. Define 
\begin{align*}
g_A(x) & = \frac{B(1+(\beta- \alpha)x)}{2\beta} + \frac{A(1-(\beta-\alpha)x)}{2\alpha} \\
& = \frac{B}{2\beta} + \frac{A}{2\alpha} + \left( \frac{B}{2\beta} - \frac{A}{2\alpha} \right) (\beta - \alpha)x.
\end{align*}
$g_A(x)$ is an increasing function of $x$ as long as $0<A<\alpha$. It is also an increasing function of $A$  when $x<1$ and $\alpha < 1/2$.

Now, let us evaluate $E[N_{\theta^*}(j+1)|\ba(j)]$. For the three cases we have 
\begin{enumerate}[(i)]
\item When $u(j) = u(\theta^*), $ \\
$E[N_{\theta^*}(j+1)|\ba(j)] = P_1(j) g_{p_{j+1}} \left( \frac{\tau_1(j) - a_{u(\theta^*)}(j)}{1-a_{u(\theta^*)}(j)} \right) + P_2(j) g_0 \left( \frac{\tau_2(j) - a_{u(\theta^*)}(j)}{1-a_{u(\theta^*)}(j)} \right). $
\item When $u(j) > u(\theta^*), $ \\
$E[N_{\theta^*}(j+1)|\ba(j)] = P_1(j) g_{0} \left( \frac{-\tau_1(j) - a_{u(\theta^*)}(j)}{1-a_{u(\theta^*)}(j)} \right) + P_2(j) g_0 \left( \frac{\tau_2(j) - a_{u(\theta^*)}(j)}{1-a_{u(\theta^*)}(j)} \right). $
\item When $u(j) < u(\theta^*), $ \\
$E[N_{\theta^*}(j+1)|\ba(j)] = P_1(j) g_{p_{j+1}} \left(  \frac{\tau_1(j) - a_{u(\theta^*)}(j)}{1-a_{u(\theta^*)}(j)} \right) + P_2(j) g_{p_{j+1}} \left( \frac{-\tau_2(j) - a_{u(\theta^*)}(j)}{1-a_{u(\theta^*)}(j)} \right)$.
\end{enumerate}

We will now bound $E[N_{\theta^*}(j+1)|\ba(j)]$ for all the three cases. We will use the fact that for all $0<a<1$, we have $\frac{\tau-a}{1-a} \leq \tau$ and $-\left( \frac{\tau+a}{1-a} \right) \leq -\tau $.

Starting with case (i), we have 
$\tau_2(j) - a_{u(\theta^*)}(j) = a_{u(\theta^*)}(j) - \tau_1(j)$. Thus, 

\begin{align*}
E[N_{\theta^*}(j+1)|\ba(j)] & = P_1(j) g_{p_{j+1}} \left(  \frac{\tau_1(j) - a_{u(\theta^*)}(j)}{1-a_{u(\theta^*)}(j)} \right) + P_2(j) g_{0} \left( \frac{a_{u(\theta^*)}(j) - \tau_1(j)}{1-a_{u(\theta^*)}(j)} \right), \\
& \leq P_1(j) g_{p_{j+1}} \left(  \frac{\tau_1(j) - a_{u(\theta^*)}(j)}{1-a_{u(\theta^*)}(j)} \right) + P_2(j) g_{p_{j+1}} \left( \frac{a_{u(\theta^*)}(j) - \tau_1(j)}{1-a_{u(\theta^*)}(j)} \right), \\
& = \frac{q_{j+1}}{2\beta} + \frac{p_{j+1}}{2\alpha} + \left( \frac{q_{j+1}}{2\beta} - \frac{p_{j+1}}{2\alpha} \right)(\beta - \alpha)\frac{\tau_1(j) - a_{u(\theta^*)}(j)}{1-a_{u(\theta^*)}(j)}\left( P_1(j) - P_2(j) \right),\\
& = \frac{q_{j+1}}{2\beta} + \frac{p_{j+1}}{2\alpha} + \left( \frac{q_{j+1}}{2\beta} - \frac{p_{j+1}}{2\alpha} \right)(\beta - \alpha)\frac{\tau_1(j) - a_{u(\theta^*)}(j)}{1-a_{u(\theta^*)}(j)}\frac{\tau_2(j) - \tau_1(j)}{\tau_2(j) + \tau_1(j)}, \\
& = \frac{q_{j+1}}{2\beta} + \frac{p_{j+1}}{2\alpha} + \left( \frac{q_{j+1}}{2\beta} - \frac{p_{j+1}}{2\alpha} \right)(\beta - \alpha)\frac{\tau_1(j) - a_{u(\theta^*)}(j)}{1-a_{u(\theta^*)}(j)}\frac{2a_{u(\theta^*)}(j) - 2\tau_1(j)}{\tau_2(j) + \tau_1(j)}, \\
& \leq \frac{q_{j+1}}{2\beta} + \frac{p_{j+1}}{2\alpha}.
\end{align*}

In case (ii),
\begin{align*}
E[N_{\theta^*}(j+1)|\ba(j)] & \leq   P_1(j) g_{0}(-\tau_1(j)) + P_2(j) g_0(\tau_2(j)),  \\
& = P_1(j) \left(\frac{1}{2\beta} - \frac{1}{2\beta}(\beta - \alpha)\tau_1(j)\right) + P_2(j)\left(\frac{1}{2\beta} + \frac{1}{2\beta}(\beta - \alpha)\tau_2(j)\right), \\
& = \frac{1}{2\beta} \ \left( \text{because $-P_1\tau_1(j) + P_2\tau_2(j) = 0 $} \right), \\
& \leq \frac{q_{j+1}}{2\beta} + \frac{p_{j+1}}{2\alpha} \ \left( \text{because } \alpha \leq \frac{1}{2} \right).
\end{align*}

In case (iii),
\begin{align*}
E[N_{\theta^*}(j+1)|\ba(j)] & \leq P_1(j) g_{p_{j+1}}(\tau_1(j)) + P_2(j) g_{p_{j+1}}(-\tau_2(j)), \\
& = \frac{q_{j+1}}{2\beta} + \frac{p_{j+1}}{2\alpha} + \left( \frac{q_{j+1}}{2\beta} + \frac{p_{j+1}}{2\alpha} \right) (\beta - \alpha)(P_1(j)\tau_1 - P_2(j)\tau_2), \\
& = \frac{q_{j+1}}{2\beta} + \frac{p_{j+1}}{2\alpha}.
\end{align*}

Thus, in all the cases $E[N_{\theta^*}(j+1)|\ba(j)] \leq \frac{q_{j+1}}{2\beta} + \frac{p_{j+1}}{2\alpha} \leq  \frac{Q_e}{2\beta} + \frac{P_e}{2\alpha}$.

Substituting this in the inequality for $E[M_{\theta^*}(T)]$, we get
$E[M_{\theta^*}(T)] =  M_{\theta^*}(0)\left\lbrace \frac{Q_e}{2\beta} + \frac{P_e}{2\alpha} \right\rbrace^T.$ Hence, 
\begin{align*}
P(|\theta_T - \theta^*| > \epsilon) & \leq M_{\theta^*}(0)\left\lbrace \frac{q}{2\beta} + \frac{p}{2\alpha} \right\rbrace^T ,\\
&  \leq \frac{1-p_{1}^{-1}\epsilon}{p_{1}^{-1}\epsilon} \left\lbrace \frac{q}{2\beta} + \frac{p}{2\alpha} \right\rbrace^T ,\\
& = \frac{p_1 - \epsilon}{\epsilon} \left\lbrace \frac{q}{2\beta} + \frac{p}{2\alpha} \right\rbrace^T. 
\end{align*}

By construction, $\hat{\theta}_{T} \geq \theta_T$. If $\hat{\theta}_{T} > \theta_T$, then $\theta^* > \hat{\theta}_{T} > \theta_T$. Thus, 
$$P(|\hat{\theta}_{T} - \theta^*| > \epsilon) \leq \frac{p_1 - \epsilon}{\epsilon} \left\lbrace \frac{q}{2\beta} + \frac{p}{2\alpha} \right\rbrace^T.$$
\end{proof}

\revision{
\begin{proof}[Proof of Lemma~\ref{lemma:alsh-time-complexity}]
  With the upper bound $P_{max}$ on $P_e$, \alsh \  can be run with $\alpha = \sqrt{P_{max}} $ (as $P_e \leq \sqrt{P_{max}} $ ). Thus, using Theorem \ref{thm:BZError},  after $T$ iterations, 
  \begin{align*}
    P(|\hat{\theta}_{T} - \theta^*| > \epsilon) & \leq \frac{p_1 -\epsilon}{\epsilon} \ \left(\frac{P_e}{2\sqrt{P_{max}}} + \frac{1- P_e}{2(1-\sqrt{P_{max}})}\right)^{T} \\
    & \leq \frac{p_1 -\epsilon}{\epsilon} \ \left(\frac{P_{max}}{2\sqrt{P_{max}}} + \frac{1- P_{max}}{2(1-\sqrt{P_{max}})}\right)^{T} \\
    & \leq  \frac{p_1 -\epsilon}{\epsilon} \ \left(\sqrt{P_{max}} + 0.5 \right)^{T}.
  \end{align*}
  Thus, for desired confidence level $\gamma$, $T= \lceil  \log_{0.5 + \sqrt{P_{max}}}(\gamma \epsilon/(p_1 -\epsilon))   \rceil$ suffices. 
\end{proof}
}
 
\section{Additional Experimental Results}\label{sec:app_figs}

\revision{	In addition to the real dataset and one semi-synthetic setup seen in Section~\ref{sec:experiments}, we consider multiple real data and semi-synthetic setups here. In the semi-synthetic setups, we start with two different types of real data sets, one as source of real prices (the Billion Prices dataset) and the second as a source for general data-driven assortments without a compact representation (the transaction logs seen in Section~\ref{sec:experiments}). This section is organized in the following manner: (a) we first describe the Billion Prices dataset, (b) we then discuss the semi-synthetic setups and the differences between them, and (c) we report results on these setups, which complement the results in Section~\ref{sec:experiments}. These results are analogous to the trends observed in Section~\ref{sec:experiments}, and show that our algorithms are scalable with minimal loss in revenues in both the general and capacitated settings.}

\revision{
\noindent\textit{Billion Prices Dataset (BPP)}: As a real world source of prices, we use the publicly available online micro price dataset from the Billion Prices Project~\citep{IAH6Z6_2016} to generate item prices. This dataset contains daily prices for all goods sold by 7 large retailers in Latin America (3 retailers) and the USA (4 retailers) between 2007 to 2010. Among the US retailers, we use pricing data from a supermarket and an electronics retailer to generate our assortment planning instances of varying sizes. The former contains 10 million daily observations for 94,000 items and the latter contains 5 million daily observations for 30,000 items. We use prices from 50 different days when generating  instances for 50 Monte Carlo runs under different settings. 
}

\revision{
\noindent\textit{Semi-synthetic Setups}: The information unavailable in the respective dataset (e.g., purchase information to guide the choice model and frequent itemset estimation with the BPP dataset) is generated synthetically. For instance,
\begin{itemize}
	\item In Section~\ref{sec:experiments}, we have already shown how real world transaction logs were augmented with synthetic prices and utilities (and consequently the MNL model) such that the prices and utilities were negatively correlated. To be specific, the utility parameter of each product is negatively correlated with its price. Specifically, for product $i$,  $v_i = \exp{-(\delta_1 + \delta_2 p_i + G_i )} $, with $\delta_1 = 4.61, \delta_2 = 0.00461 $, and $G_i$ is Gaussian noise with mean 0 and variance 0.1.
	\item In addition to the above, we augment these transaction log datasets with synthetic prices and utilities such that each utility parameter is chosen independently from the uniform distribution $U[0,1]$ and each price is also chosen independently from $U[0,1000]$.
	\item For the BPP dataset, we generate the set of arbitrary feasible/relevant assortment uniformly at random from the set of all assortments, and apply the aforementioned ideas of generating utilities in a correlated as well as independent fashion to complete the instance descriptions.
\end{itemize} 
} 

\subsection{Experiments with General Assortments}

\revision{Here, we consider arbitrary collection of feasible assortments to optimize over. Our first result is in Figure~\ref{fig:gen-ast-real-uni}, where we have used the real world transactions logs but synthetically generated the prices and the utilities independently. This is in contrast with Figure~\ref{fig:gen-ast-real-inv} in Section~\ref{sec:experiments} where these two quantities were negatively correlated. As can be seen, the three proposed methods continue to do much better than the exhaustive search benchmark.}

\begin{figure}[ht]
	\centering
	\includegraphics[width=.48\textwidth]{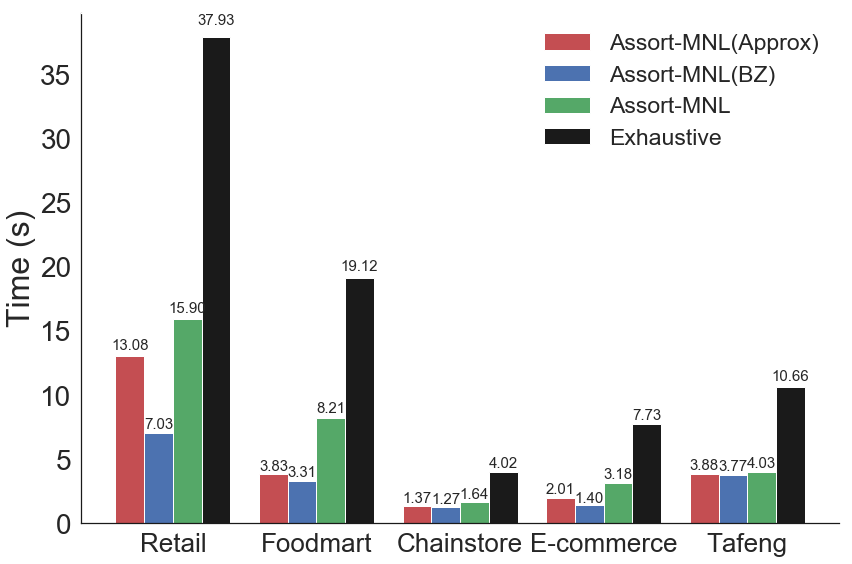}
	\includegraphics[width=.48\textwidth]{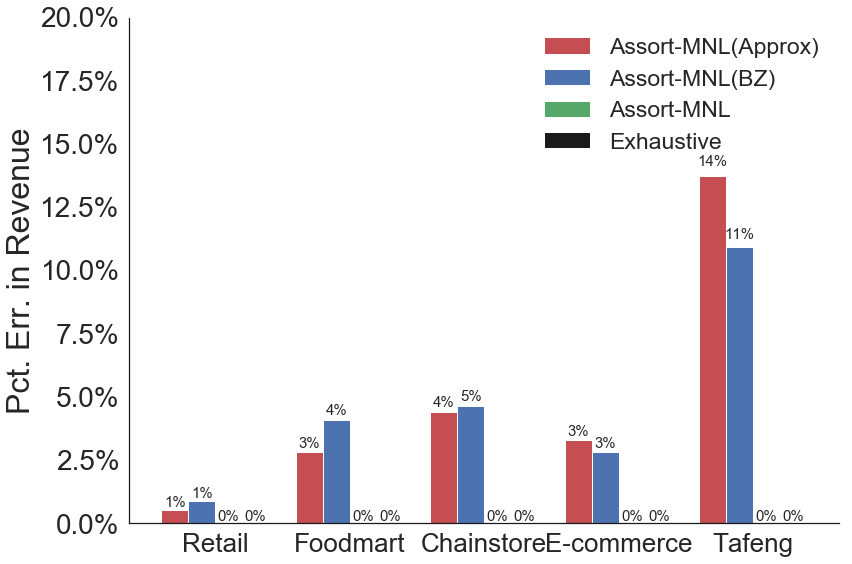}
	\caption{Performance of algorithms over general data-driven instances derived from five different frequent itemset datasets (with prices and utilities generated from a uniform distribution for the semi-synthetic datasets).  \label{fig:gen-ast-real-uni}}
\end{figure}
 
 \revision{Next we study the trends in computation time and relative revenue error as the size of assortments is varied. For this, we use the semi-synthetic instances generated from the BPP dataset as described above. While Figure~\ref{fig:bpp-gen-ast-real-price-uni} captures trends when prices are uncorrelated with utilities, Figure~\ref{fig:bpp-gen-ast-real-price-inv} captures trends when prices are negatively correlated with utilities. As can be observed from both these figures, the three proposed algorithms, namely \ann, \aheu, and \alsh{} continue to perform well across all scales. In particular, the relative revenue errors incurred by our algorithms are almost $5\times$ smaller than the error values incurred on the Ta Feng dataset (see Figure~\ref{fig:tafeng-gen-ast-real-price} in Section~\ref{sec:experiments}).} 

\begin{figure}[ht]
\centering
\includegraphics[width=.48\textwidth]{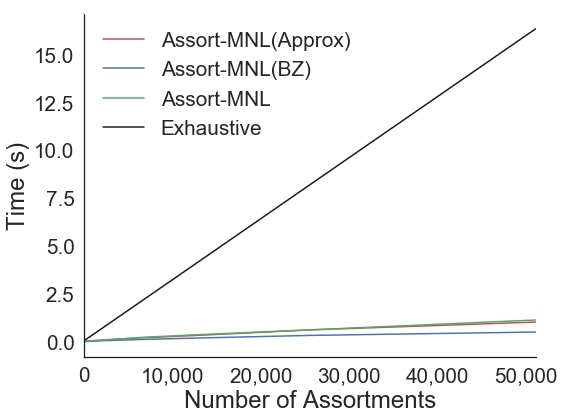}
\includegraphics[width=.48\textwidth]{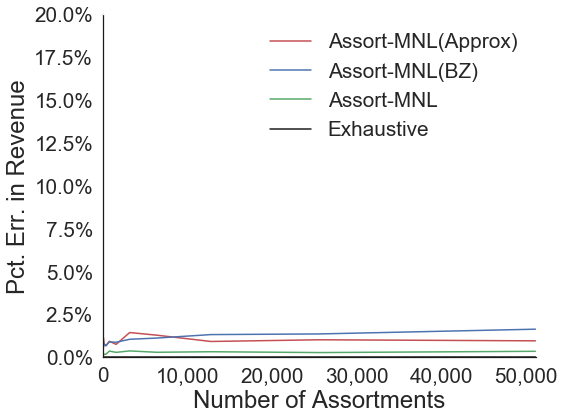}
\caption{Performance of algorithms over general data-driven instances derived from the BPP dataset (with product utilities generated from a uniform distribution). The x-axis corresponds to the number of feasible assortments.
	\label{fig:bpp-gen-ast-real-price-uni}}
\end{figure} 

 \begin{figure}[ht]
	\centering
	\includegraphics[width=.48\textwidth]{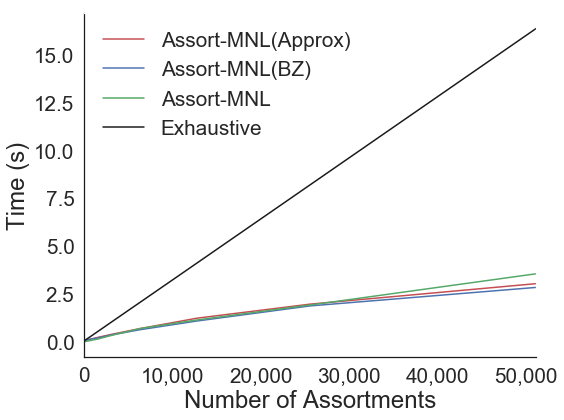}
	\includegraphics[width=.48\textwidth]{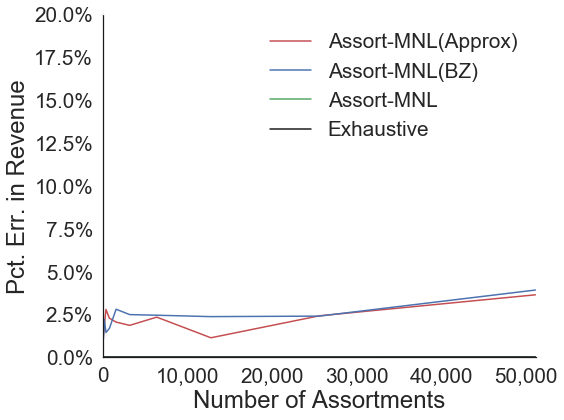}
	\caption{Performance of algorithms over general data-driven instances derived from the BPP dataset (with product utilities being negatively correlated with prices). The x-axis corresponds to the number of feasible assortments.
		\label{fig:bpp-gen-ast-real-price-inv}}
\end{figure}

\subsection{Experiments with Capacitated Assortments}

\revision{
We experiment with the semi-synthetic BPP dataset based instances in the capacitated setting. The experiment conducted here mirrors the one in Section~\ref{sec:experiments}, where the latter uses the real world Ta Feng dataset. As shown in Figures~\ref{fig:bpp-cap-real-price-prod-uni} and \ref{fig:bpp-cap-real-price-prod-inv}, the performance of \ann \ is better by an order of magnitude or more in terms of time complexity while allowing for almost zero loss in relative revenue. In Figure~\ref{fig:bpp-cap-real-price-prod-uni}, the utilities are drawn from an uniform distribution, whereas in Figure~\ref{fig:bpp-cap-real-price-prod-inv}, the utilities are negatively correlated with the prices.}

\begin{figure}[ht]
	\centering
	\includegraphics[width=.48\textwidth]{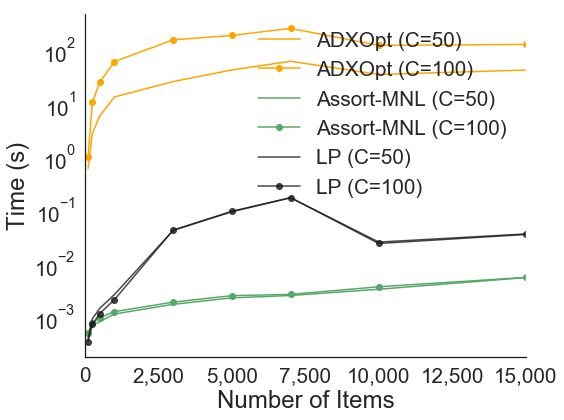}
	\includegraphics[width=.48\textwidth]{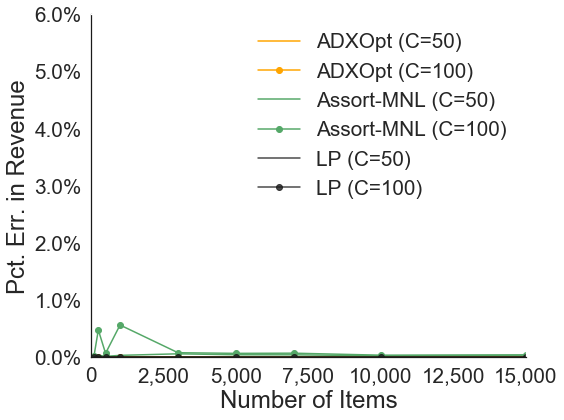}
	\caption{Performance of algorithms in the capacitated setting over instances derived from the BPP dataset  (with product utilities generated from a uniform distribution). The x-axis corresponds to the number of items.
		\label{fig:bpp-cap-real-price-prod-uni}}
\end{figure}

\begin{figure}[ht]
	\centering
	\includegraphics[width=.48\textwidth]{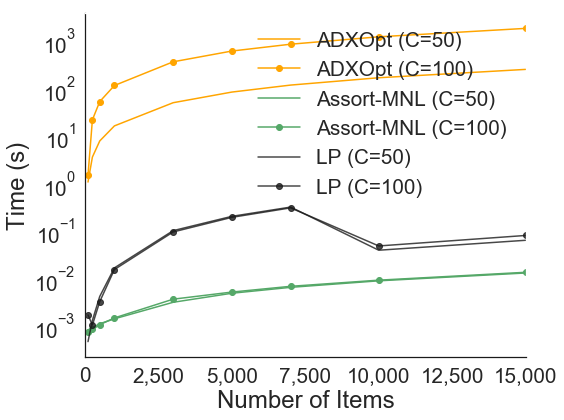}
	\includegraphics[width=.48\textwidth]{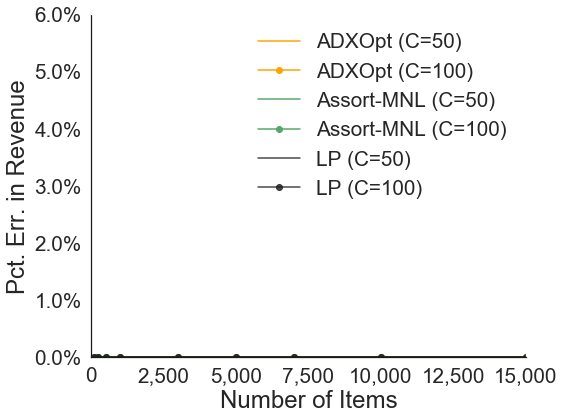}
	\caption{Performance of algorithms in the capacitated setting over instances derived from the BPP dataset  (with product utilities being negatively correlated with prices). The x-axis corresponds to the number of items.
		\label{fig:bpp-cap-real-price-prod-inv}}
\end{figure} 

\revision{
To summarize, the experiments above complement the results on real data reported in Section~\ref{sec:experiments}, and show that the proposed algorithms are indeed capable of optimizing assortments at scale, given arbitrary collections of relevant feasible assortments to optimize over.}

\end{document}